\documentclass[11pt]{article}

\usepackage{bbm}
\usepackage{amsmath,amsfonts,amssymb,enumerate}
\usepackage{color}
\usepackage{lscape}
\usepackage{latexsym}
\usepackage{multirow}
\usepackage{rotating}
\usepackage{threeparttable}
\usepackage{graphicx}
\usepackage {ulem}
\usepackage{algpseudocode,algorithm}
\usepackage[margin=1in]{geometry}

\numberwithin{equation}{section}
\newtheorem{theorem}{\bf Theorem}[section]
\newtheorem{lemma}[theorem]{\bf Lemma}
\newtheorem{proposition}[theorem]{\bf Proposition}

\def \bR {\Bbb R}
\def \bN {\Bbb N}
 
\def\DD{\mathbb{D}} 
\def\cB{{\cal{B}}}
\def\cD{{\cal{D}}}
\def\cF{{\cal{F}}}
\def\cH{{\cal{H}}}
\def\cL{{\cal{L}}}
\def\cN{{\cal{N}}}
\def\cP{{\cal{P}}}
\def\cQ{{\cal{Q}}}

\def\ba{{\bf{a}}}
\def\bbb{{\bf{b}}}
\def\bc{{\bf{c}}}
\def\bd{{\bf{d}}}
\def\ff{{\bf{f}}}
\def\bg{{\bf{g}}}
\def\bw{{\bf{w}}}
\def\be{{\bf{e}}}
\def\bx{{\bf{x}}}
\def\by{{\bf{y}}}
\def\bz{{\bf{z}}}

\def\bA{{\bf{A}}}
\def\bB{{\bf{B}}}
\def\bD{{\bf{D}}}
\def\bI{{\bf{I}}}
\def\bJ{{\bf{J}}}
\def\bS{{\bf{S}}}
\def\bP{{\bf{P}}}
\def\bQ{{\bf{Q}}}
\def\bW{{\bf{W}}}
\def\bX{{\bf{X}}}
\def\bZ{{\bf{Z}}}

\newenvironment{proof}{\noindent{\em Proof:}}{\quad \hfill$\Box$\vspace{2ex}}

\newtheorem{definition}[theorem]{\bf Definition}

\renewcommand{\theequation}{\arabic{section}.\arabic{equation}}

\renewcommand{\thefigure}{\thesection.\arabic{figure}}
\renewcommand{\thetable}{\Roman{table}}

\graphicspath{{figs/}}

\begin{document}

%%%%%%%%%%%%%%%%

\title{\bf Uniform Convergence of Deep Neural Networks with Lipschitz Continuous Activation Functions and Variable Widths}
\author{Yuesheng Xu\thanks{Department of Mathematics and Statistics, Old Dominion University, Norfolk, VA 23529, USA. E-mail address: {\it y1xu@odu.edu}. Supported in part by US National Science Foundation under grants DMS-1912958  and DMS-2208386, and by the US National Institutes of Health under grant
R21CA263876. Corresponding author.  } \ and Haizhang Zhang\thanks{
School of Mathematics (Zhuhai), Sun Yat-sen University, Zhuhai, P.R. China. E-mail address: zhhaizh2@sysu.edu.cn.
Supported in part by National Natural Science Foundation of China under grants 11971490 and 12126610.}}

\date{}
\maketitle

%\subjclass[2000]{Primary 46C15; Secondary 46E22}
%
%\date{}
%
%\dedicatory{}

\begin{abstract}
We consider deep neural networks with a Lipschitz continuous activation function and with weight matrices of variable widths. We establish a uniform convergence analysis framework in which sufficient conditions on weight matrices and bias vectors together with the Lipschitz constant are provided to ensure uniform convergence of the deep neural networks to a meaningful function as the number of their layers tends to infinity. In the framework, special results on uniform convergence of deep neural networks with a fixed width, bounded widths and unbounded widths are presented. In particular,  
as convolutional neural networks are special deep neural networks with weight matrices of increasing widths, we put forward conditions on the mask sequence which lead to uniform convergence of resulting convolutional neural networks.
The Lipschitz continuity assumption on the activation functions allows us to include in our theory most of commonly used activation functions in applications.

%We explore convergence of deep neural networks with the popular ReLU activation function, as the depth of the networks tends to infinity. To this end, we introduce the notion of activation domains and activation matrices of a ReLU network. By replacing applications of the ReLU activation function by multiplications with activation matrices on activation domains, we obtain an explicit expression of the ReLU network. We then identify the convergence of the ReLU networks as convergence of a class of infinite products of matrices. Sufficient and necessary conditions for convergence of these infinite products of matrices are studied. As a result, we establish necessary conditions for ReLU networks to converge that the sequence of weight matrices converges to the identity matrix and the sequence of the bias vectors converges to zero as the depth of ReLU networks increases to infinity. Moreover, we obtain sufficient conditions in terms of the weight matrices and bias vectors at hidden layers for pointwise convergence of deep ReLU networks.  These results provide mathematical insights to the design strategy of the well-known deep residual networks in image classification.

\medskip

\end{abstract}


%\end{document}

\section{Introduction}

The last decade has witnessed the immense success of deep learning \cite{Goodfellow, LeCun}. As we all know, a major part of such successes are due to the powerful expressiveness of the deep neural networks in representing a function. In other words, the deep neural network is the engine of deep learning. In order to explain why deep learning works so well, there is need to understand deep neural networks from rigorous mathematical viewpoints. 
As deep neural networks have advanced in machine learning, they have gained much attention in the applied mathematics community \cite{Daubechies, Devore1, E, Montanelli1, Montanelli2, Shen1, Shen2, Shen3, Xu2022, Xu2023, Zhou1} and have gone beyond machine learning. 
Deep neural networks, considered as a function class to represent or approximate a function, have exhibited  superiority in many aspects to classical approximation functions such as polynomials, trigonometric polynomials, splines, finite elements, wavelets, and kernel functions in approximation and numerical analysis. Unlike the classical function approximation in which a function is approximated by a linear combination of basis functions, deep neural networks approximate a given function by compositions of functions with a networks. Specifically, a neural network is a vector-valued function defined through consecutive function compositions of a given activation function with parameters consisting of weight matrices and bias vectors. A deep neural network of a given function may be determined by finding the parameters that minimize the difference between it and the given function. Mathematically, one would expect that as the number of layers of the deep neural network increases, the difference diminishes and eventually goes to zero as the number tends to infinity. In a special case when the activation function is the rectified linear unit (ReLU), this question was investigated in a number of studies \cite{Daubechies, Shen3, Zhou1}. 

A closely related mathematical question, even more basic, is when a deep neural network converges to a meaningful function as its layer number tends to infinity. This question was studied recently for the ReLU activation function with a fixed width in \cite{XuZhang2021}, for the ReLU activation function with a convolution network in \cite{XuZhang2022} and for contractive activation functions with a fixed width in \cite{HuangXuZhang}. Although the contractivity hypothesis covers interesting activation functions such as sigmoid, there are many activation functions frequently used in applications that are not contractive, for example, ReLU, parametric rectified linear unit (PReLU), exponential linear unit (ELU) and scaled exponential linear unit (SELU), to name a few. Therefore, there is a need to understand uniform convergence of deep neural networks defined by a non-contractive activation function.  This paper will study uniform convergence of deep neural networks of Lipschitz continuous activation functions with weight matrices of variable widths (including bounded and unbounded widths).

Main difficulty in analyzing convergence of deep neural networks is a result of the nonlinearity of the activation function. This was overcome in \cite{XuZhang2021} for the ReLU activation function by re-expressing the functional application of the activation function in terms of matrix-vector multiplication with activation matrices. Pointwise convergence of neural networks with the ReLU activation function was then analyzed by using the matrix-vector multiplication. When a general Lipschitz continuous activation function is chosen, one can take the advantage of its Lipschitz continuity to overcome the difficulty caused by its nonlinearity. We propose a condition that intertwines the Lipschitz constant of the activation function with the weight matrices to ensure the uniform convergence of the resulting neural networks. The main contribution of this work lies in laying out a general framework for uniform convergence analysis of deep neural networks with general activation functions and pooling operators, both of which are Lipschitz continuous.

%Before entering the formal analysis, we make a few comments on the scientific values and potential applications of our results. 
Understanding conditions that ensure convergence of deep neural networks is not only theoretically interesting, 
but also practically advantageous in guiding their training in applications. 
As we know, a deep learning model usually possesses a large number of hidden layers and a massive amount of parameters. For example, Residual Networks (ResNets) can reach over 1,000 layers \cite{KaimingHe}. A ResNet with only 50 layers has over 23 million parameters, and the overwhelming ChatGPT (GPT-3) model has approximately 175 billion parameters. In applications, the parameters of a DNN are determined via a training process by minimizing a loss function on given training data. %When the DNN attains certain accuracy on the training data, the system is considered to be convergent. This convergence is different from the definition of convergence of a sequence of functions in analysis. 
It is desirable to figure out whether or not a DNN system with so many parameters can eventually converge to a meaningful function in a rigorous mathematical sense. Mathematical conditions on the parameter that ensure convergence of the DNNs as the number of their layers increases will be beneficial to the interpretability of the DNNs. Such conditions can also be helpful in training a DNN. For instance, if such conditions are available, people can generate initial parameters of a DNN that satisfy or nearly satisfy the conditions. With such an initial deployment of parameters, the DNN will be inclined to converge more quickly. Finally, uniform convergence rates of DNNs will be applicable to mathematical analysis on the generalization ability of DNN models \cite{HuangYang}.

We organize this paper in seven sections. In section 2, we describe the setting of deep neural networks. Section 3 is devoted to developing a general framework for uniform convergence analysis of deep neural
networks with pooling. A key ingredient that ensures uniform convergence of the deep neural networks is a condition that  intertwines the Lipschitz constants of the activation function and the pooling operator with the norm of the weight matrices. In sections 4 and 5, we present uniform convergence results for deep neural networks with weight matrices of  fixed widths and bounded widths, respectively. While in section 6, we consider deep neural networks with weight matrices of unbounded widths. Finally, in section 7 we present uniform convergence theorems for convolutional neural networks.

\section{Deep Neural Networks}

%\section{Deep Neural Networks and Convergence}
\setcounter{equation}{0}
In this section, we recall the definition of deep neural networks.

We now describe deep neural networks with width $m_n\in\bN$ at the $n$-th level, for $n\in\bN$, from input space $\bR^s$ to the output space $\bR^t$, where $s, t\in \bN$.
For a given univariate activation function $\sigma: \bR\to \bR$,
we define the vector-valued function
\begin{equation}\label{activationF}
    \sigma(\bx):=[\sigma(x_1), \sigma(x_2), \dots, \sigma(x_s)]^\top, \ \ \mbox{for}\ \ \bx:=[x_1,x_2,\dots,x_s]^\top\in \bR^s.
\end{equation}
As in \cite{XuZhang2021}, for $n$ vector-valued functions $f_j$, $j\in\bN_n:=\{1,2,\dots,n\}$, such that the range of $f_i$ is contained in the domain of $f_{i+1}$, $i\in\bN_{n-1}$, the consecutive composition of $f_i$, $i\in\bN_n$, is denoted by
\begin{equation}\label{consecutive_composition}
    \bigodot_{i=1}^n f_i:=f_n\circ f_{n-1}\circ\cdots\circ f_2\circ f_1,
\end{equation}
whose domain is that of $f_1$.
For each $n\in\bN$, by $\bW_n$ we denote the weight matrix, and by $\bbb_n$ the bias vector, of the $n$-th hidden layer. Clearly,  $\bW_n\in\bR^{m_n\times m_{n-1}}$, for $n\in\bN$ with $m_0:=s$, and $\bbb_n\in\bR^{m_n}$ for $n\in\bN$. 
The deep neural network is a function defined by
\begin{equation}\label{DNN}
    \cN_n(\bx):=\left(\bigodot_{i=1}^n \sigma(\bW_i \cdot+\bbb_i)\right)(\bx),\ \ \bx\in\bR^s.
\end{equation}
Given the weight matrix $\bW_{o}\in \bR^{t\times m_n}$ and bias vector  $\bbb_{o}\in\bR^t$, of the output layer, the output of the DNN is then given by
$$
\by=\bW_{o}\cN_n(\bx)+\bbb_{o},\ \ \bx\in\bR^s.
$$
We are concerned with convergence of the functions $\cN_n$ determined by the deep neural network as $n$ increases to infinity. Because the output layer is a linear function of $\cN_n$ and thus, it does not affect the convergence. Hence, we will consider convergence of the function sequence $\cN_n$, as $n\to\infty$.
For a fixed $n\in \bN$, $\cN_n$ is a vector-valued function.

It is convenient to express the neural network $\cN_n$ in recursion in $n$.
From \eqref{DNN} and the definition \eqref{activationF}, we have the recursion
\begin{equation}\label{Step1}
    \cN_1(\bx):=\sigma(\bW_1 \bx+\bbb_1)
\end{equation}
and
\begin{equation}\label{Recursion}
    \cN_{n+1}(\bx)=\sigma(\bW_{n+1}\cN_n(\bx)+\bbb_{n+1}), \ \ \bx\in \bR^s, \ \ \mbox{for all} \ \ n\in \bN.
\end{equation}

The goal of this paper is to understand what conditions are required for the weight matrices $\bW_n$ and the bias vectors $\bbb_n$ to ensure convergence of the deep neural network for a general activation function.
For this purpose, we suppose that $\sigma$ is Lipschitz continuous with the Lipschitz constant $L$. That is,
\begin{equation}\label{Lip-cont-1D}
    |\sigma(x)-\sigma(y)|\leq L|x-y|, \ \ \mbox{for all}\ \ x, y\in \bR.
\end{equation}
When $L< 1$, $\sigma$ is contractive and when $L=1$, $\sigma$ is non-expansive. Convergence of deep neural networks with a non-expansive activation function and uniform convergence of deep neural networks with a contractive activation function were established in a recent paper \cite{HuangXuZhang}. We are particularly interested in understanding conditions that ensure uniform convergence of deep neural networks with a Lipschitz continuous activation function with $L>1$.
Many commonly used activation functions are Lipschitz continuous. Among them, some have their Lipschitz constants $L>1$. For example, the parametric rectified linear unit (PReLU) \cite{He2015}
$$
\sigma(x):=\begin{cases}
\alpha x,\ \ & x< 0,\\
x,\ \ & x\geq 0
\end{cases}
$$
is Lipschitz continuous with the Lipschitz constant $L:=\max\{\alpha, 1\}$. When $\alpha>1$, PReLU is expanding.
The exponential linear unit (ELU) \cite{Clevert} is defined by
$$
\sigma(x):=
\begin{cases}
\alpha(e^x-1),\ \ & x\leq 0,\\
x,\ \ & x> 0.
\end{cases}
$$
It can be verified that ELU is Lipschitz continuous with the Lipschitz constant $L:=\max\{\alpha, 1\}$. When $\alpha>1$, ELU is expanding. Moreover,
the scaled exponential linear unit (SELU) was proposed in \cite{Klambauer} to construct self-normalized neural networks and it has the form
$$
\sigma(x):=\lambda
\begin{cases}
\alpha(e^x-1),\ \ & x<0,\\
x,\ \ & x\geq 0,
\end{cases}
$$
with parameter $\lambda =1.0507$ and $\alpha =1.67326$. Clearly, SELU is Lipschitz continuous with the Lipschitz constant $L:=\lambda\alpha$, which is greater than 1. Hence, SELU is again expanding.

%Sigmoid linear unit (SiLU):
%$$
%\sigma(x):=\frac{x}{1+e^{-x}}, \ \ x\in \bR.
%$$

Convergence of deep neural networks is measured by a vector norm.
We say that  a vector norm $\|\cdot\|$ on $\bR^l$ satisfies the extension invariant condition if
\begin{equation}\label{vectornorm-zero}
    \|\tilde{\ba}\|=\|\ba\|,\ \ \mbox{for}\ \ \tilde{\ba}:=
\left[\begin{array}{c}
    \ba
    \\
    {\bf 0}
    \end{array}\right],
     \ \mbox{with}\ \ba\in \bR^\nu,\ {\bf 0}\in \bR^{l-\nu} \ \mbox{for any}\ \nu \in\bN_l
\end{equation}
and satisfies the monotonicity condition if
\begin{equation}\label{nondecreasingvectornorm}
    \|\ba\|\le\|\bbb\| \mbox{ whenever }|a_i|\le |b_i|,\ i\in\bN_l,\  \mbox{for}\ \ba:=[a_1,\dots,a_l]^T, \bbb:=[b_1,\dots,b_l]^T\in\bR^l.
\end{equation}
For each $\ba\in \bR^l$, we recall the $\ell_p$-norms as
$$
\|\ba\|_p:=\left(\sum_{j=1}^l|a_j|^p \right)^\frac{1}{p}, \ \ \mbox{for}\ \ 1\leq p<+\infty
$$
and
$$
\|\ba\|_\infty:=\max\{|a_j|: j\in \bN_l\}, \ \ \mbox{for}\ \ p=+\infty.
$$
It can be confirmed that the $\ell_p$-norms for all $1\le p\le +\infty$ satisfy both of these conditions. We also need a matrix norm $\|\cdot\|$ on $\bR^{k\times l}$, which we require to be induced by a vector norm, that is,
$$
\|\bA\|=\sup_{\bx\in\bR^l,x\ne0}\frac{\|\bA\bx\|}{\|\bx\|},\ \ \mbox{for}\ \ \bA\in\bR^{k\times l}.
$$
Clearly, this matrix norm has the property that
\begin{equation}\label{matrixcon1}
\|\bA\bB\|\le \|\bA\|\|\bB\|, \ \mbox{ for all}\ \bA, \bB \ \mbox{with}\ \bA\bB \ \mbox{well-defined}.
\end{equation}
%and
%\begin{equation}
%\|I_i\|\le 1 \mbox{ for each }I_i\in \cD_m.
%\end{equation}
%Note that the Frobenius norm satisfies (\ref{matrixcon1}) but does not satisfy (\ref{matrixcon2}).

It follows from \eqref{activationF}  and \eqref{nondecreasingvectornorm} that
\begin{equation}\label{Lip-cont}
    \|\sigma(\bx)-\sigma(\by)\|\leq L\|\bx-\by\|, \ \ \mbox{for all}\ \ \bx,\by\in \bR^s.
\end{equation}
When $\sigma$ is Lipschitz continuous with the Lipschitz constant $L$, the neural network $\cN_n$ is also Lipschitz continuous with the Lipschitz constant $L^n\prod_{j=1}^n\|\bW_j\|$. In fact, from the recursion \eqref{Recursion} and the Lipschitz continuity \eqref{activationF}, we observe that
$$
\|\cN_n(\bx)-\cN_n(\by)\|\leq \left(L^n\prod_{j=1}^n\|\bW_j\|\right)\|\bx-\by\|, \ \ \mbox{for all}\ \ \bx,\by\in \bR^s.
$$

%In the remaining part of this section, we consider uniform convergence of deep neural networks with a Lipschitz continuous activation function.
When the weight matrices $\bW_n$ have variable widths, the resulting neural networks $\cN_n(\bx)$ have variable dimensions. Considering convergence of such a sequence of neural networks requires special care. We first assume that the weight matrices $\bW_n$ have a fixed width $l\in\bN$.
In this case, we say that the deep neural network $\cN_n$ defined by \eqref{DNN} via $\bW_n$, $\bbb_n$, $n\in\bN$, and an activation function $\sigma$ converges uniformly in a bounded set $\DD\subseteq \bR^s$ to a limit function $\cN: \DD\to \bR^l$ if for any $\epsilon>0$, there exists $N\in \bN$ such that 
$$
\sup_{\bx\in \DD}\|\cN_n(\bx)-\cN(\bx)\|<\epsilon, \ \ \mbox{whenever} \ \ n\geq N.
$$
By $C_l(\DD)$, we denote the space of continuous vector-valued functions $\cN: \DD\to \bR^l$ defined on $\DD\subseteq \bR^s$ with 
$$
\sup\{\|\cN(\bx)\|: \bx\in\DD\}<+\infty,
$$
where $\|\cdot\|$ is a vector norm on $\bR^l$. For any vector norm, $C_l(\DD)$ is complete.

We now consider the case when the matrix widths $m_n$, $n\in \bN$, of deep neural networks are variable but bounded. Specifically, we let $l:=\max\{m_n: n\in \bN\}$. Then, we have that $l<+\infty$,  $1\leq m_n\leq l$ for all $n\in\bN$, and $m_{n_0}=l$ for some $n_0\in \bN$. We define the deep neural network $\cN_n(\bx)$ by \eqref{DNN}. Clearly,  $\cN_n(\bx)$ is a vector-valued function in $\bR^{m_n}$. The size of $\cN_n(\bx)$ varies according to $n$. 
Considering convergence of such a sequence requires us to extend the weight matrices $\bW_n\in \bR^{m_n\times m_{n-1}}$ and bias vectors $\bbb_n\in \bR^{m_n}$ to $\tilde\bW_n\in \bR^{l\times l}$ and $\tilde\bbb_n\in \bR^l$, respectively.

When the widths $m_n$, $n\in\bN$, of neural networks are unbounded, there exists a subsequence $m_{n_i}$, $i\in \bN$, with $\lim_{i\to +\infty}m_{n_i}=+\infty$. Due to the unboundedness of the widths, we will extend all vectors and matrices to elements in sequence spaces $\ell_p(\bN)$ and $\ell_p(\bN^2)$, respectively, and consider uniform convergence of the neural networks that result from the extension in the sequence spaces.
We now review the notion of sequence spaces. 
For $1\leq p\leq +\infty$, by $\ell_p(\bN)$ we denote the space of sequences $\bx$ with $\|\bx\|_{\ell_p(\bN)}<+\infty$, where
$$
\|\bx\|_{\ell_p(\bN)}:=\left(\sum_{j=1}^\infty|x_j|^p\right)^\frac{1}{p}, \ \ \mbox{for}\ \ 1\leq p<+\infty
$$ 
and
$$
\|\bx\|_{\ell_\infty(\bN)}:=\sup\{|x_j|: j\in \bN\}.
$$ 
Throughout this paper, we reserve $\|\cdot\|_p$ as the vector norm in $\bR^\mu$ for a $\mu<\infty$.
We also need the notion of spaces of semi-infinite matrices. For $1\leq p\leq +\infty$, we let $\ell_p(\bN^2):=\ell_p(\bN\times \bN)$ denote the spaces of semi-infinite matrices $\bW: \ell_p(\bN)\to\ell_p(\bN)$, viewed as operators, with $\|\bW\|_{\ell_p(\bN^2)}<+\infty$, where $\|\cdot\|_{\ell_p(\bN^2)}$ are operator norms induced from the $\ell_p(\bN)$ norms. 
We extend matrices $\bW_n\in \bR^{m_n\times m_{n-1}}$ to $\tilde{\bW}_n\in \ell_p(\bN\times\bN_s)$ for $n=1$ and $\tilde{\bW}_n\in \ell_p(\bN^2)$ for $n>1$, and extend vectors $\bbb_n\in\bR^{m_n}$ to $\tilde\bbb_n\in \ell_p(\bN)$. For $1\leq p\leq +\infty$, we also define $C_{\ell_p(\bN)}(\DD)$ to be the space of continuous sequence-valued functions $\cN: \DD\to \ell_p(\bN)$ defined on $\DD\subseteq \bR^s$ with 
$$
\sup\{\|\cN(\bx)\|_{\ell_p(\bN)}: \bx\in\DD\}<+\infty.
$$

When the matrix widths $m_n$, $n\in\bN$, of neural networks are variable, either bounded or unbounded, 
we define $\tilde\cN_n(\bx)$ by \eqref{DNN} with $\bW_n$ and $\bbb_n$ being replaced by $\tilde\bW_n$ and $\tilde\bbb_n$, respectively. 
The neural networks $\tilde\cN_n(\bx)$ are extensions of $\cN_n(\bx)$ and satisfy the same recursion as \eqref{Step1} and \eqref{Recursion} with  $\tilde\bW_n$ and $\tilde\bbb_n$. We then consider uniform convergence of the sequence $\tilde\cN_n(\bx)$ in either $C_l(\DD)$ or $C_{\ell_p(\bN)}(\DD)$. We say the sequence $\cN_n(\bx)$ converges uniformly in $C_l(\DD)$ (resp. $C_{\ell_p(\bN)}(\DD)$) if $\tilde\cN_n(\bx)$ converges uniformly to $\cN\in C_l(\DD)$ (resp. $\cN\in C_{\ell_p(\bN)}(\DD)$) when the matrix widths are bounded (resp. unbounded). Note that when the widths are not fixed, the uniform convergence depends on how the weight matrices and bias vectors are extended. Unambiguous  extensions of the weight matrices will be specified in later sections.

%The goal of this section is to establish convergence theorems of $\cN_n(\bx)$. The approach used in the last section for convergence analysis is limited to deep neural networks with weight matrices of a {\it fixed} width. 

%%%%%%%%%%%%%%%%%%%%%%%%%%%%

%\section{DNNs with Nonexpansive Activation Functions} \label{sec:sepcialpiecewise}
%%%%%%%%%%%%%%%%%%%%%%

%%%%%%%%%%%%%%%%%%%%%%%%%%
\section{Uniform Convergence Analysis Framework}
%%%%%%%%%%%%%%%%%%%%%%%%%%

This section is devoted to establishing an analysis framework for uniform convergence of deep neural networks with pooling in a general setting. We will apply this framework to various scenarios in later sections to produce convergence results for deep neural networks of various types.

We now describe the setting. Let $l,\mu\in\bN$.
By $\hat\bW_n$, $n\in \bN$, we denote a sequence of weight matrices in $\bR^{(l+\mu)\times l}$ (resp.  $\ell_p(\bN^2)$) and by $\hat\bbb_n$, $n\in \bN$, a sequence of bias vectors in $\bR^l$ (resp. $\ell_p(\bN)$). Let $\cP_\mu: \bR^{l+\mu}\to \bR^l$ (resp. $\ell_p(\bN)\to\ell_p(\bN)$) denote a pooling operator (linear or nonlinear). We assume that the pooling operator $\cP_\mu$ is Lipschitz continuous with the Lipschitz constant $P$. % and satisfies the condition $\cP_\mu({\bf 0})={\bf 0}$. 
When $\cP_\mu$ is linear, its Lipschitz constant is identical to its norm. Specific pooling operators will be discussed later.
Suppose that $\cN_n\in \bR^l$, $n\in \bN$, satisfy the recursion
\begin{equation}\label{Step1-hat}
    \hat\cN_1(\bx):=\sigma(\cP_\mu(\hat\bW_1 \bx)+\hat\bbb_1)
\end{equation}
and
\begin{equation}\label{Recursion-hat}
    \hat\cN_{n+1}(\bx)=\sigma(\cP_\mu(\hat\bW_{n+1}\hat\cN_n(\bx))+\hat\bbb_{n+1}), \ \ \bx\in \bR^s, \ \ \mbox{for all} \ \ n\in \bN,
\end{equation}
associated with the weight matrices $\hat\bW_n$, the bias vectors $\hat\bbb_n$, and an activation function $\sigma$. If $\cP_\mu$ is the identity operator, then $\mu=0$ and there is no pooling.
When $\hat\bW_n:\bR^{l+\mu}\to\bR^l$, $\hat\bbb_n\in \bR^l$, $n>1$, we have that $\hat\cN_n(\bx)\in \bR^l$, and when $\hat\bW_n:\ell_p(\bN)\to\ell_p(\bN)$, $\hat\bbb_n\in \ell_p(\bN)$, $n>1$, we have that $\hat\cN_n(\bx)\in \ell_p(\bN)$. The goal of this section is to establish results on uniform convergence of $\hat\cN_n(\bx)$, $n\in \bN$. 

Throughout this section, we assume that the norms for appropriate vectors/sequences/matrices involved are well-defined and  the neural networks $\hat\cN_n(\bx)$, $n\in \bN$, satisfy the recursion \eqref{Step1-hat} and \eqref{Recursion-hat} without further mentioning. We adopt the convention for the notation of the product of numbers
$$
%\prod_{j=\mu}^\nu a_j:=a_\nu a_{\nu+1}\cdots a_\mu, \ \ \mbox{if}\ \mu\leq\nu, \ \ \mbox{and}\ \ 
\prod_{j=\mu}^\nu a_j:=1, \ \ \mbox{if}\ \ \mu>\nu.
$$
The next lemma prepares us for convergence analysis of neural networks. To this end, for $m,k\in\bN$, we define $e_{m,k}:=\|\hat\bbb_{m+k}-\hat\bbb_k\|$ and $E_{m,k}:= \|\hat\bW_{m+k}-\hat\bW_{k}\|$.

\begin{lemma}\label{ThreeTerms-hat} If the activation function $\sigma$ and the pooling operator $\cP_\mu$ are Lipschitz continuous with the Lipschitz  constants $L$ and $P$, respectively, then
for all $n, m\in\bN$,
\begin{align}\label{Basic-Estimate-hat}
    \|\hat\cN_{n+m}(\bx)-\hat\cN_n(\bx)\|\leq&L \sum_{i=0}^{n-1}\Lambda_{n+m}^{i-1}e_{m,n-i}+LP\sum_{i=0}^{n-2}\Lambda_{n+m}^{i-1}\|\hat\cN_{n-1-i}(\bx)\|E_{m,n-i} 
    \nonumber\\
   &+LP\Lambda_{n+m}^{n-2}\|\hat\bW_{m+1}\hat\cN_m(\bx)-\hat\bW_1\bx\|, \ \ \bx\in \bR^s,
\end{align}
where
\begin{equation}\label{Lambda}
    \Lambda_n^i:=\prod_{j=0}^{i}LP\|\hat\bW_{n-j}\|, \ \ n\in\bN, \ i<n.
\end{equation}
\end{lemma}
\begin{proof}
We prove inequality \eqref{Basic-Estimate-hat} by induction on $n$. We first consider the case $n=1$. Since both the activation function  $\sigma$ and the pooling operator are Lipschitz continuous with the Lipschitz constants $L$ and $P$, respectively, by recursion \eqref{Step1-hat} and \eqref{Recursion-hat}, and the Lipschitz conditions, we obtain for all  $m\in\bN$ that
$$
\|\hat\cN_{m+1}(\bx)-\hat\cN_1(\bx)\|\leq Le_{m,1} +LP\|\hat\bW_{m+1}\hat\cN_m(\bx)-\hat\bW_1\bx\|.
$$
That is, inequality \eqref{Basic-Estimate-hat} holds for $n=1$. We now assume that inequality \eqref{Basic-Estimate-hat} holds for $n=k$. 
%That is, for all positive integer $m$ we have that
%\begin{align}\label{InductionH}
%    \|\hat\cN_{k+m}(\bx)-\hat\cN_k(\bx)\|\leq& L\sum_{i=0}^{k-1}\left(\prod_{j=0}^{i-1}LP\|\hat\bW_{k+m-j}\|\right)\|\hat\bbb_{k+m-i}-\hat\bbb_{k-i}\|\nonumber\\ 
%   & +\sum_{i=0}^{k-2}\|\hat\cN_{k-1-i}(x)\|\left(LP\prod_{j=0}^{i-1}LP\|\hat\bW_{k+m-j}\|\right)\|\hat\bW_{k+m-i}-\hat\bW_{k-i}\|\nonumber\\
%   &+\left(LP\prod_{j=0}^{k-2}LP\|\hat\bW_{k+m-j}\|\right)\|\hat\bW_{m+1}\hat\cN_m(\bx)-\hat\bW_1\bx\|.
%\end{align}
We next proceed for the case $n=k+1$. Again, by the Lipschitz continuity of the activation function $\sigma$ and the pooling operator $\cP_\mu$ with the Lipschitz constants $L$ and $P$, respectively, and by recursion \eqref{Recursion-hat}, for all $m\in\bN$, we obtain that 
\begin{align}\label{CaseOfk+1}
    \|\hat\cN_{k+1+m}(\bx)-\hat\cN_{k+1}(\bx)\|\leq& Le_{m,k+1}
    +LP\|\hat\cN_k(x)\|E_{m,k+1}+LP\|\hat\bW_{k+1+m}\|\|\hat\cN_{k+m}(\bx)-\hat\cN_k(\bx)\|.
\end{align}
Substituting the induction hypothesis into the third term on the right-hand-side of inequality \eqref{CaseOfk+1} and noting for all $i,k,m\in\bN$ that
$
\Lambda_{k+m}^{i-1}=\prod_{j=1}^{i}\|\hat\bW_{k+1+m-j}\|,
$
we find for all $m\in\bN$ that
\begin{align}\label{CaseOfk+1*}
    \|\hat\cN_{k+1+m}(\bx)-\hat\cN_{k+1}(\bx)\|\leq& Le_{m,k+1}
    +LP\|\hat\cN_k(x)\|E_{m,k+1}+L \sum_{i=0}^{k-1}\Lambda_{k+1+m}^ie_{m,k-i}\nonumber\\ 
    &+LP\sum_{i=0}^{k-2}\Lambda_{k+1+m}^i\|\hat\cN_{k-1-i}(x)\|E_{m,k-i}\nonumber\\ 
    &+LP\Lambda_{k+1+m}^{k-1}\|\hat\bW_{m+1}\hat\cN_m(\bx)-\hat\bW_1\bx\|.
\end{align}
By replacing $i+1$ with $i$,  the third and fourth terms of the right-hand-side of inequality \eqref{CaseOfk+1*} become
$$
L\sum_{i=1}^{k}\Lambda_{k+1+m}^{i-1}e_{m,k+1-i}
\ \
\mbox{and}\ \
LP\sum_{i=1}^{k-1}\Lambda_{k+1+m}^{i-1}\|\hat\cN_{k+1-1-i}(x)\|E_{m,k+1-i},
$$
respectively. Substituting them into the third and fourth terms of the right-hand-side of 
\eqref{CaseOfk+1*} leads to
inequality \eqref{Basic-Estimate-hat} with $n=k+1$. Thus, the induction principle ensures that inequality \eqref{Basic-Estimate-hat} holds for all positive integers $n$ and $m$.
\end{proof}

Lemma \ref{ThreeTerms-hat} has pointed the direction for establishing the analysis framework for uniform convergence of deep neural networks. 
Along this line, we now derive two additional technical lemmas that help estimate the products and sums appearing in \eqref{Basic-Estimate-hat}.

\begin{lemma}\label{Two_Bounds}
If $\alpha_n$, $n\in\bN$, is a sequence of non-negative numbers satisfying the condition
\begin{equation}\label{less-than1}
    \lim_{n\to \infty}\alpha_n=\alpha<1,
\end{equation}
then there exists a positive constant $c$ such that 
\begin{equation}\label{Prod1ton}
\prod_{j=1}^{n}\alpha_j\leq c, \ \ \mbox{for all}\ \ n\in \bN
\end{equation}
and
\begin{equation}\label{sum-prod}
    \sum_{j=1}^{n}
\left(\prod_{i=j+1}^{n}\alpha_i\right)\leq c, \ \ \mbox{for all}\ \ n\in \bN.
\end{equation}
\end{lemma}
\begin{proof}
We introduce the notation
$$
A_n:=\prod_{j=1}^{n}\alpha_j,
$$
and
$$
B_n:=\sum_{j=1}^{n}
\left(\prod_{i=j+1}^{n}\alpha_i\right).
$$
Since condition \eqref{less-than1} is satisfied, there exist $N\in \bN$ and $\alpha\leq \gamma<1$ such that 
\begin{equation}\label{lessthangamma}
\alpha_n\leq \gamma, \ \ \mbox{for all} \ \ n\geq N.
\end{equation}
The sequence $A_{n}$, $n\in\bN$, is decreasing if $n$ is sufficiently large. Therefore, it must be bounded and thus, \eqref{Prod1ton} holds true.

It remains to show that there exists a positive constant $c$ such that $B_n\leq c$ for all $n\in \bN$. Let $N\in\bN$ be the integer chosen so that inequality \eqref{lessthangamma} holds.
For all $n\geq N$, we write
\begin{equation}\label{Write-inTwoSums}
B_n=\sum_{j=1}^{N-1}
\left(\prod_{i=j+1}^{n}\alpha_i\right)+\sum_{j=N}^{n}
\left(\prod_{i=j+1}^{n}\alpha_i\right).
\end{equation}
For the first sum of the right-hand-side of equation \eqref{Write-inTwoSums}, by \eqref{lessthangamma} we have that
\begin{align*}
    \sum_{j=1}^{N-1}
\left(\prod_{i=j+1}^{n}\alpha_i\right)&=\sum_{j=1}^{N-1}
\left(\prod_{i=j+1}^{N-1}\alpha_i\right)\left(\prod_{i=N}^{n}\alpha_i\right)\\
&\leq \gamma^{n-N+1}\sum_{j=1}^{N-1}
\left(\prod_{i=j+1}^{N-1}\alpha_i\right).
\end{align*}
The inequality above together with $0<\gamma<1$ leads to
\begin{equation}\label{TheFirstSum}
    \sum_{j=1}^{N-1}
\left(\prod_{i=j+1}^{n}\alpha_i\right)< B_{N-1}.
\end{equation}
For the second sum of the right-hand-side of equation \eqref{Write-inTwoSums}, once again, according to \eqref{lessthangamma}, we have that
\begin{align*}
\sum_{j=N}^{n}
\left(\prod_{i=j+1}^{n}\alpha_i\right)&=1+
\sum_{j=N}^{n-1}
\left(\prod_{i=j+1}^{n}\alpha_i\right)\\
&
\leq \sum_{j=0}^{n-N}\gamma^j<\frac{1}{1-\gamma}.
\end{align*}
Substituting this inequality and the estimate \eqref{TheFirstSum} into the right-hand-side of equation \eqref{Write-inTwoSums} leads to 
$$
B_n<B_{N-1}+\frac{1}{1-\gamma}, \ \ \mbox{for all}\ \ n\geq N.
$$
Defining
$$
c:=\max\left\{B_1, B_2, \dots, B_{N-2}, B_{N-1}+\frac{1}{1-\gamma}\right\},
$$
we conclude that $B_n\leq c$, for all $n\in \bN$.
\end{proof}

Here comes the second technical lemma.

\begin{lemma}\label{Zero-Limit-Lemma}
Let $\alpha_n$, $\beta_n$, $n\in \bN$, be two sequences of non-negative numbers. If 
\begin{equation}\label{Two-Limits}
  \lim_{n\to\infty} \alpha_n<1, \ \ \mbox{and} \ \ \lim_{n\to\infty}\beta_n=0,
\end{equation}
then for any positive $\epsilon>0$, there exits $N\in\bN$ such that 
\begin{equation}\label{Zero-Limit}
\sum_{i=0}^{n-1}\left(\prod_{j=0}^{i-1}\alpha_{n-j}\right)\beta_{n-i}<\epsilon, \ \ \mbox{for all}\ \ n>N.
\end{equation}
\end{lemma}
\begin{proof}
For $n\in\bN$, we define the sequence
$$
A_n:= \sum_{i=0}^{n-1}\left(\prod_{j=0}^{i-1}\alpha_{n-j}\right)\beta_{n-i}.
$$
Let $\epsilon>0$ be arbitrary. By the second limit of \eqref{Two-Limits}, there exists $N_1\in \bN$ such that  $\alpha_n\le 1$ for all $n\ge N_1$ and 
\begin{equation}\label{Beta}
    \beta_n<\epsilon,\ \ \mbox{for all}\ \ n\geq N_1. 
\end{equation}
We split $A_n$ into two terms according to $N_1$. That is,
\begin{equation}\label{TWO-SUMS}
    A_n= \sum_{i=0}^{n-N_1}\left(\prod_{j=0}^{i-1}\alpha_{n-j}\right)\beta_{n-i}+ \sum_{i=n-N_1+1}^{n-1}\left(\prod_{j=0}^{i-1}\alpha_{n-j}\right)\beta_{n-i}.
\end{equation}
We denote by $A^1_n$ and $A_n^2$ the first and second sums of \eqref{TWO-SUMS}, respectively. By using the inequality \eqref{Beta}, we obtain for all $n\geq N_1$ that
$$
A_n^1<\epsilon\sum_{i=0}^{n-N_1}\left(\prod_{j=0}^{i-1}\alpha_{n-j}\right).
$$
By changing the index $k:=n-j$, we have that
$$
\left(\prod_{j=0}^{i-1}\alpha_{n-j}\right)=\left(\prod_{k=n-i+1}^n\alpha_k\right).
$$
Hence, we get that
\begin{equation}\label{S1}
A_n^1<\epsilon\sum_{i=0}^{n-N_1}\left(\prod_{k=n-i+1}^n\alpha_k\right).
\end{equation}
Letting $j:=n-i$, using the non-negativity of the numbers $\alpha_k$ and employing  \eqref{sum-prod} of Lemma \ref{Two_Bounds},  there exists a positive constant $c_1$ such that for all $n\geq N_1$,
\begin{align*}
    \sum_{i=0}^{n-N_1}\left(\prod_{k=n-i+1}^n\alpha_k\right)&=\sum_{j=N_1}^n\left(\prod_{k=j+1}^n\alpha_k\right)\\
    &\leq \sum_{j=1}^n\left(\prod_{i=j+1}^n\alpha_i\right)\leq c_1.
\end{align*}
This together with \eqref{S1} yields that $A_n^1<c_1\epsilon$ for all $n\geq N_1$.

We next estimate $A_n^2$. Clearly, there exists a positive constant $c_2$ such that $\beta_n\leq c_2$ for all $n\in \bN$.  Using this fact and the change of indices $k:=n-j$, we obtain that 
\begin{align*}
    A_n^2&\leq c_2\sum_{i=n-N_1+1}^{n-1}\left(\prod_{k=n-i+1}^{n}\alpha_k\right).
\end{align*}
We then perform an additional change of indices $j:=n-i+1$ and find that
\begin{align}\label{A_n2}
    A_n^2&\leq 
    c_2\sum_{j=2}^{N_1}\left(\prod_{k=j}^{n}\alpha_k\right).
\end{align}
The first limit of \eqref{Two-Limits} implies that there exist $0<\gamma<1$ and $N_2\in \bN$ with $N_2\geq N_1$ such that $\alpha_n\leq \gamma$ for all $n> N_2$. It follows from \eqref{A_n2} that
\begin{align*}
    A_n^2&\leq c_2\sum_{j=2}^{N_1}\left(\prod_{k=j}^{N_2-1}\alpha_k\right)\left(\prod_{k=N_2}^{n}\alpha_k\right)\\
    &\leq c_2\gamma^{n-N_2+1}\sum_{j=2}^{N_1}\left(\prod_{k=j}^{N_2-1}\alpha_k\right)\\
    &=c_2\gamma^{n-N_2+1}\sum_{j=1}^{N_1-1}\left(\prod_{k=j+1}^{N_2-1}\alpha_k\right).
\end{align*}
Again, by the non-negativity of the numbers $\alpha_k$ and estimate \eqref{sum-prod} of Lemma \ref{Two_Bounds}, we observe for all $n>N_2$ that
$$
\sum_{j=1}^{N_1-1}\left(\prod_{k=j+1}^{N_2-1}\alpha_k\right)
\leq \sum_{j=1}^n\left(\prod_{k=j+1}^n\alpha_k\right)\leq c_1.
$$
%\textcolor{red}{(Note: For the first inequality above to be true, we should also require $N_1$ to satisfy $\alpha_j\le 1$ for $j\ge N_1$.)}
This implies for all $n>N_2$ that
$$
A_n^2\leq c_1c_2\gamma^{n-N_2+1}. 
$$
Since $0<\gamma<1$, there exists $N_3\in \bN$ such that for all $n>N_3$, $\gamma^n<\epsilon$. Therefore, we conclude that $A_n^2<c_1c_2\epsilon$, for all $n>N_2+N_3$. Adding this estimate to the estimate for $A_n^1$, we have that 
$$
A_n<(c_1+c_1c_2)\epsilon,\ \  \mbox{for all} \ \ n>N_2+N_3,
$$
proving  \eqref{Zero-Limit}.
\end{proof}

With the help of Lemmas \ref{ThreeTerms-hat} and \ref{Zero-Limit-Lemma}, we establish the following result on uniform convergence of the sequence $\hat\cN_n(\bx)$ generated by the recursion \eqref{Step1-hat} and \eqref{Recursion-hat} in a general setting. We say that $\hat\bbb_n$, $n\in\bN$, converges if there exists $\bbb^*\in \bR^l$ (resp. $\bbb^*\in \ell_p(\bN)$) such that $\lim_{n\to\infty}\|\hat\bbb_n-\bbb^*\|=0$ when $\hat\bbb_n\in \bR^l$ (resp. $\hat\bbb_n\in \ell_p(\bN)$), $n\in\bN$, and $\|\cdot\|$ denotes a norm in $\bR^l$ (resp. in $\ell_p(\bN)$). Likewise, we say that $\hat\bW_n$, $n\in\bN$, converges if there exists $\bW^*\in \bR^{l\times l}$ (resp. $\bW^*\in \ell_p(\bN^2)$) such that $\lim_{n\to\infty}\|\hat\bW_n-\bW^*\|=0$ when $\hat\bW_n\in \bR^{l\times l}$ (resp. $\hat\bW_n\in \ell_p(\bN^2)$), $n\in\bN$, and $\|\cdot\|$ denotes a norm in $\bR^{l\times l}$ (resp. in $\ell_p(\bN^2)$). Here, we do not distinguish the notation of a vector norm from a matrix norm  since they can be clearly differentiated from the context.  It is known that the sequence $\hat\bbb_n$, $n\in \bN$, converges if and only if it is Cauchy in $\bR^l$, and  the sequence $\hat\bW_n$, $n\in \bN$, converges if and only if it is Cauchy in $\ell_p(\bN^2)$.

\begin{theorem}\label{Fixed-widths-hat}
Suppose that the activation function $\sigma$ and the pooling operator $\cP_\mu$ are Lipschitz continuous with the Lipschitz  constants $L$ and $P$, respectively,
and that $\DD\subset \bR^s$ is bounded. % by $D>0$. 
If the sequences $\hat\bbb_n$, $\hat\bW_n$, $n\in\bN$, converge,
the sequence $\|\hat\bW_n\|$, $n\in\bN$, satisfies the condition
\begin{equation}\label{BasicAssumption-hat}
  \omega:=  \lim_{n\to\infty}LP\|\hat\bW_n\|<1,
\end{equation}
and there exists a positive constant $\rho$ such that 
\begin{equation}\label{UniformBoundednessofDNN}
    \sup_{\bx\in\DD}\|\hat\cN_n(\bx)\|\leq \rho, \ \ \mbox{for all}\ \  n\in\bN,
\end{equation}
then the sequence $\hat\cN_n$ converges  uniformly to $\cN\in C_l(\DD)$ (resp.  $\cN\in C_{\ell_p(\bN)}(\DD)$) if $\hat\bW_n\in\bR^{l\times l}$ (resp.  $\hat\bW_n\in\ell_p(\bN^2)$).
\end{theorem}
\begin{proof}
%Let $\DD\subseteq\bR^s$ be a bounded set.
It suffices to show that the sequence $\cN_n$, $n\in\bN$, is Cauchy in space $C_l(\DD)$ or $C_{\ell_p(\bN)}(\DD)$. Let $\epsilon>0$ be arbitrary. We wish to prove that there exist $N\in \bN$ and $c>0$ such that
\begin{equation}\label{Cauchy-hat}
 \sup_{\bx\in\DD}   \|\hat\cN_{n+m}(\bx)-\hat\cN_n(\bx)\|<c\epsilon, \ \  \mbox{for all}\ \ %\textcolor{red}{\sout{n,m\geq N\ } 
 n\ge N,m\in\bN.
\end{equation}
Motivated by Lemma \ref{ThreeTerms-hat}, we introduce
$$
J_{n,1}:=L\sum_{i=0}^{n-1}\Lambda_{n+m}^{i-1}e_{m,n-i},\ \
J_{n,2}(\bx):=LP\sum_{i=0}^{n-2}\Lambda_{n+m}^{i-1}\|\hat\cN_{n-1-i}(\bx)\|E_{m,n-i}, \ \ \bx\in \DD,
$$
and
$$
J_{n,3}(\bx):=LP\Lambda_{n+m}^{n-2}\|\hat\bW_{m+1}\hat\cN_m(\bx)-\hat\bW_1\bx\|, \ \ \bx\in \DD.
$$
%\textcolor{red}{\sout{Because of \eqref{Lip-cont-1D}},} 
By Lemma \ref{ThreeTerms-hat}, for all $n,m\in\bN$
we have that
\begin{equation}\label{ThreeTermE-hat}
  \sup_{\bx\in\DD}   \|\hat\cN_{n+m}(\bx)-\hat\cN_n(\bx)\|\leq J_{n,1}+ \sup_{\bx\in\DD}   J_{n,2}(\bx)+ \sup_{\bx\in\DD}   J_{n,3}(\bx).
\end{equation}
We next show that the three terms on the right-hand-side of inequality \eqref{ThreeTermE-hat} tend to zero as $n\to\infty$ for all $m\in\bN$. 

We first consider $J_{n,1}$. For $n\in\bN$, we let
$$
\alpha_n:=LP\|\hat\bW_{n+m}\|\ \ \mbox{and}\ \ \beta_n:=e_{m,n}, \ \ m\in\bN.
$$
By hypothesis \eqref{BasicAssumption-hat}, we have that $\lim_{n\to\infty}\alpha_n=\omega<1$.
Since the sequence $\hat\bbb_n$, $n\in\bN$, converges, we note that $\lim_{n\to\infty}\beta_n=0$.  Applying Lemma \ref{Zero-Limit-Lemma} and noting the definition \eqref{Lambda} of $\Lambda_{n+m}^{i-1}$, we conclude that there exists $N_1\in\bN$ such that
\begin{equation}\label{J_n1}
J_{n,1}<L\epsilon, \ \ \mbox{for all}\ \  n>N_1, \ m\in\bN.
\end{equation}

%Moreover, by hypothesis \eqref{BasicAssumption-hat}, we observe that $\alpha_n:=LP\|\hat\bW_n\|\to \omega<1$. Applying Lemma \ref{Zero-Limit-Lemma}, we conclude that there exists $N_1\in\bN$ such that
%\begin{equation}\label{J_n1}
%J_{n,1}<\epsilon, \ \ \mbox{for all}\ \  n>N_1, \ m\in\bN.
%\end{equation}

For the second term on the right-hand-side of \eqref{ThreeTermE-hat}, by the uniform boundedness assumption \eqref{UniformBoundednessofDNN} on $\hat\cN_n(\bx)$, we see that
$$
\sup_{\bx\in\DD}J_{n,2}(\bx)\leq \rho LP\sum_{i=0}^{n-2}\Lambda_{n+m}^{i-1}E_{m,n-i}, \ \ m\in\bN.
$$
For $n\in\bN$, we define
$$
\alpha_n:=LP\|\hat\bW_{n+m}\|\ \ \mbox{and}\ \ \beta_n:=E_{m,n}, \ \ m\in\bN.
$$ 
By  hypothesis \eqref{BasicAssumption-hat}, we observe that $\lim_{n\to\infty}\alpha_n<1$ and by the convergence of sequence $\hat\bW_n$, $n\in\bN$, we find that $\lim_{n\to\infty}\beta_n=0$.
Again, by employing Lemma \ref{Zero-Limit-Lemma}, there exists $N_2\in\bN$ such that
\begin{equation}\label{J_n2}
    \sup_{\bx\in\DD}J_{n,2}(\bx)<\rho LP\epsilon, \ \ \mbox{for all}\ \ n>N_2.
\end{equation}
%$J_{n,2}(\bx)\to 0$, as $n\to\infty$, for all $m\in\bN$ and $\bx\in \DD$.

It remains to estimate the third term on the right-hand-side of \eqref{ThreeTermE-hat}. To this end, we assume that $\DD\subset \bR^s$ is bounded by $D>0$. Again, due to the uniform boundedness of $\|\hat\cN_n(\bx)\|$, $n\in\bN$, and convergence of $\hat\bW_n$, $n\in\bN$,  which implies that $\|\hat\bW_n\|\leq w$ for all $n\in\bN$ and for some $w>0$, we reach that
\begin{equation}\label{J_n3}
\sup_{\bx\in\DD}J_{n,3}(\bx)\leq LPw(\rho+D)\Lambda_{n+m}^{n-2}.
\end{equation}
By the definition \eqref{Lambda} of $\Lambda_{n+m}^{n-2}$ with the change of indices $k:=n+m-j$, we note that
\begin{equation}\label{ProductIdentity}
   \Lambda_{n+m}^{n-2}=\prod_{k=m+2}^{n+m}LP\|\hat\bW_k\|.
\end{equation}
By virtue of hypothesis \eqref{BasicAssumption-hat}, there exist $N_3\in\bN$ and $\omega_0\in[\omega, 1)$ such that $LP\|\hat\bW_k\|<\omega_0$, for all $k>N_3$. 
Therefore, by \eqref{J_n3} and \eqref{ProductIdentity}, 
for all $m\ge1$ and $n\geq 2$,
\begin{equation}\label{J_n3-new}
\sup_{\bx\in\DD}J_{n,3}(\bx)\leq LP\omega(\rho+D)\omega_0^{n+m-N_3}\omega^{N_3-m-1}.. 
\end{equation}
%\textcolor{red}{Note: $m$ cannot be larger than $N_3$. It is a fixed integer that can be small. Thus, the right-hand-side of (3.27) can be changed to
%$$
%LP\omega(\rho+D)\omega_0^{n+m-N_3}\omega^{N_3-m-1}.
%$$}
In view of $0<\omega_0<1$, there exists $N_4\in\bN$ such that $\omega_0^{n-1}<\epsilon$, for all $n>N_4$. Upon substituting this result into the right-hand-side of \eqref{J_n3-new} yields
\begin{equation}\label{J_n3-newNew}
\sup_{\bx\in\DD}J_{n,3}(\bx)\leq LPw(\rho+D)\epsilon, \ \ \mbox{for all}\ \ %\textcolor{red}{\sout{m>N_3}\ 
m\ge1, \ n>N_4. 
\end{equation}
Now, we choose $N:=\max\{N_i: i=1,2,3,4\}$. Upon substituting 
estimates \eqref{J_n1}, \eqref{J_n2} and \eqref{J_n3-newNew} into the right-hand-side of \eqref{ThreeTermE-hat} gives rise to \eqref{Cauchy-hat} with $c:=L+\rho LP+LPw(\rho+D)$.
%
%Due to hypothesis \eqref{BasicAssumption-hat}, there exists $N\in\bN$ such that
%$$
%\prod_{j=0}^{n-2}LP\|\hat\bW_{n+m-j}\|=\prod_{j=m+2}^{n+m}LP\|\hat\bW_j\|\leq \omega^{n-1} \leq\epsilon, \ \ \mbox{for all}\ \ n\geq N, m\in\bN.
%$$
%This together with the boundedness of $\|\hat\cN_m(\bx)\|$ implies that there exists a positive constant $c$ such that, for all $n\geq N$, $m\in\bN$ and $\bx\in \DD$,
%$$
%J_{n,3}(\bx)\leq LP\left(\prod_{j=m+2}^{n+m}LP\|\hat\bW_j\|\right)[\|\hat\bW_{m+1}\|\|\hat\cN_m(\bx)\|+\|\hat\bW_1\|\|\bx\|]\leq c\epsilon.
%$$
%
%Finally, let $\epsilon>0$ be arbitrary. By \eqref{ThreeTermE-hat}, we observe that there exist a positive constant $c$ and $N\in\bN$ such that for all $n\geq N$, $m\in\bN$ and $\bx\in \DD$,
%\begin{equation*}
%     \|\hat\cN_{n+m}(\bx)-\hat\cN_n(\bx)\|\leq J_1+J_2(\bx)+J_3(\bx)<c\epsilon.
%\end{equation*}
%That is, $\hat\cN_n(\bx)$, $n\in\bN$, is a Cauchy sequence in $C_l(\DD)$ (resp. $C_{\ell_p(\bN)}(\DD)$), if $\hat\bW_n\in \bR^{l\times l}$ (resp. $\hat\bW_n\in \ell_p(\bN^2)$). Because the spaces  $C_l(\DD)$ and $C_{\ell_p(\bN)}(\DD)$ are Banach, the sequence  $\hat\cN_n(\bx)$, $n\in\bN$, converges uniformly to a function $\cN\in C_l(\DD)$ (resp. $\cN\in C_{\ell_p(\bN)}(\DD)$).
\end{proof}

It was indicated in \cite{HuangXuZhang} that when the activation function is contractive, the neural networks
converge exponentially as the layer number tends to infinity if both the sequence of weight matrices and that of bias vectors converge exponentially. Our next task is to establish an exponential convergence theorem for the deep neural network with general Lipschitz continuous activation functions, not necessarily contractive, and general Lipschitz continuous  poolings.
%To this end, for $\nu\in \bN_0$ we define
%\begin{equation}\label{limits}
%    \omega_\nu:=\lim_{m\to\infty}\prod_{j=0}^\nu L\|\bW_{n+m-j}\|
%\end{equation}
%and $w_{-1}:=1$.
%When the hypothesis of Theorem \ref{Fixed-widths-hat} is satisfied, the sequence defined by \eqref{limits} is well-defined. We have the following result.

\begin{lemma}\label{LemmaA}
Suppose that the activation function $\sigma$ and the pooling operator $\cP_\mu$ are Lipschitz continuous with the Lipschitz  constants $L$ and $P$, respectively,
and that $\DD\subset \bR^s$ is bounded by $D>0$.
If  $\lim_{n\to\infty}\hat\bbb_n=\bbb^*$, $\lim_{n\to\infty}\hat\bW_n=\bW^*$, and satisfies \eqref{BasicAssumption-hat}, and there exists a positive constant $\rho$ such that \eqref{UniformBoundednessofDNN} holds, then there exists a function $\cN\in C_l(\DD)$ (resp. $\cN\in C_{\ell_p(\bN)}(\DD)$) if $\hat\bW_n\in\bR^{l\times l}$ (resp. $\hat\bW_n\in \ell_p(\bN^2)$) such that for all $n\in\bN$, 
\begin{equation}\label{Bound2}
\sup_{\bx\in\DD}    \|\hat\cN_n(\bx)-\cN(\bx)\|\leq L\left[\sum_{i=0}^{n-1}\omega_0^{i}e_{n-i}
    +\rho LP\sum_{i=0}^{n-2}\omega_0^{i}E_{n-i}+LPw(\rho+D)\omega_0^{n-1}
    \right], 
\end{equation}
where $e_{n}:=\|\hat\bbb_{n}-\bbb^*\|$, $E_{n}:=\|\hat\bW_{n}-\bW^*\|$, $w$ is an upper bound of the sequence $\|\hat\bW_n\|$, $n\in\bN$, and $\omega\leq\omega_0<1$. 
\end{lemma}
\begin{proof}
Theorem \ref{Fixed-widths-hat} ensures that the deep neural networks $\hat\cN_n(\bx)$ converges to a function $\cN$. By Lemma \ref{ThreeTerms-hat}, for all positive integers $n$, $m$, inequality \eqref{Basic-Estimate-hat} holds. Due to hypothesis \eqref{BasicAssumption-hat}, there exist $N\in\bN$ and $\omega\leq \omega_0<1$ such that 
$LP\|\hat\bW_m\|\leq \omega_0$, for all $m\geq N$.
It follows that
$\Lambda_{n+m}^{i-1}\leq \omega_0^i$,  $i=0,1,\dots, n-1$, for all $m\geq N, n\in\bN$
and $\Lambda_{n+m}^{n-2}\leq \omega_0^{n-1}$, for all $m\geq N, n\in\bN$.
Substituting these bounds into the right-hand-side of the inequality \eqref{Basic-Estimate-hat},
letting $m\to \infty$ in the both sides of the resulting inequality, and using the limits of $\hat\bbb_n$ and $\hat\bW_n$, and the uniform boundedness of $\hat\cN_n(\bx)$, we obtain for $\bx\in \DD$ the error bound \eqref{Bound2}.
\end{proof}

Lemma \ref{LemmaA} allows us to establish the following exponential convergence result.

\begin{theorem}\label{Order-Fixed-widths-hat}
Suppose that the activation function $\sigma$ and the pooling operator $\cP_\mu$ are Lipschitz continuous with the Lipschitz  constants $L$ and $P$, respectively, and $\DD\subseteq \bR^s$ is bounded.
%by $D>0$. 
If $\hat\bbb_n$ and $\hat\bW_n$ converge, respectively, to $\bbb^*$ and $\bW^*$ exponentially with \eqref{BasicAssumption-hat}, and there exists a positive constant $\rho$ such that $\|\hat\cN_n(\bx)\|\leq \rho$, for all $\bx\in\DD$ and all $n\in\bN$, then the neural networks $\hat\cN_n$ converge to a function 
$\cN\in C_l(\DD)$ (resp. $\cN\in C_{\ell_p(\bN)}(\DD)$) if $\hat\bW_n\in\bR^{l\times l}$ (resp. $\hat\bW_n\in \ell_p(\bN^2)$)
exponentially and uniformly in $\DD$.
\end{theorem}
\begin{proof}
First of all, Theorem \ref{Fixed-widths-hat} ensures that the neural networks $\hat\cN_n$ converge to a function $\cN$ uniformly in $\DD$. It remains to prove that the convergence rate of $\hat\cN_n$ to $\cN$ is exponential.
For this purpose, by the hypothesis of this theorem, we conclude that there exist $0<r<1$ and $c>0$ such that 
$e_n\leq cr^n$
and 
$E_n\leq cr^n.$
Substituting these estimates into the right-hand-side of  \eqref{Bound2} in Lemma \ref{LemmaA}, we find that
$$
\sup_{\bx\in\DD}\|\hat\cN_n(\bx)-\cN(\bx)\|\leq (cL+c\rho L^2P)nr_0^{n}+\frac{L^2Pw(\rho+D)}{\omega_0}r_0^{n},
$$
where $r_0:=\max\{r, \omega_0\}$ and $D>0$ is a bound of $\DD$. 
Thus,  there exist $N\in\bN$ and $r_1\in(r_0,1)$ such that for all $n>N$,
$$
\sup_{\bx\in\DD}\|\hat\cN_n(\bx)-\cN(\bx)\|\leq \left[cL+c\rho L^2P+\frac{L^2Pw(\rho+D)}{\omega_0}\right]r_1^{n}.
$$
That is, the neural networks $\hat\cN_n$ converge to the function $\cN$ exponentially and uniformly in $\DD$.
\end{proof}

In Theorems \ref{Fixed-widths-hat} and \ref{Order-Fixed-widths-hat}, we assume that the sequence $\|\hat\cN_n(\bx)\|$ is uniformly bounded in a bounded set $\DD$. This can be derived from the next lemma and an additional hypothesis on the vector/sequence norm. In the next lemma, we assume that either $\bW_n\in\bR^{(m_n+\mu)\times m_{n-1}}$ with $m_0:=s$, and $\bbb_n\in\bR^{m_n}$, $n\in\bN$, or $\bW_1\in\ell_p(\bN\times \bN_s)$, $\bW_n\in\ell_p(\bN^2)$, $n>1$, and $\bbb_n\in\ell_p(\bN)$, $n\in\bN$. 

\begin{lemma}\label{BOUND-hat} If the activation function $\sigma$ and the pooling operator $\cP_\mu$ are Lipschitz continuous with the Lipschitz constants $L$ and $P$, respectively, %and $\cP_\mu({\bf 0})={\bf 0}$,
%the weight matrices $\bW_i\in\bR^{m_i\times m_{i-1}}$, for $i\in\bN_n$ with $m_0:=s$, and the bias vectors $\bbb_i\in\bR^{m_i}$ for $i\in\bN_n$, 
then for all $n\in\bN$
\begin{equation}\label{cN-bound-hat}
    \|\hat\cN_n(\bx)\|\leq \left(\prod_{j=1}^nLP\|\hat\bW_j\|\right)\|\bx\|+\sum_{j=1}^n
    \left(\prod_{i=j+1}^nLP\|\hat\bW_i\|\right)(L\|\hat\bbb_j\|+\|(\sigma\circ\cP_\mu)({\bf 0}_j)\|),
\end{equation}
where ${\bf 0}_j\in \bR^{m_j+\mu}$ are zero vectors for the finite dimensional case and ${\bf 0}_j\in \ell_p(\bN)$ are zero vectors for the infinite dimensional case.
\end{lemma}
\begin{proof}
We establish inequality \eqref{cN-bound-hat} by induction on $n$. When $n=1$, by \eqref{Step1}, for all $\bx\in \bR^s$ we have that
\begin{align*}
    \|\hat\cN_1(\bx)\|&\leq \|\sigma(\cP_\mu(\hat\bW_1\bx)+\hat\bbb_1)-\sigma(\cP_\mu({\bf 0}_1))\|+\|\sigma(\cP_\mu({\bf 0}_1))\|\\
    &\leq LP\|\hat\bW_1\|\|\bx\|+L\|\hat\bbb_1\|+\|\sigma(\cP_\mu({\bf 0}_1))\|.
\end{align*}
%where ${\bf 0}_1\in \bR^{m_1}$ is a zero vector.
Thus, inequality \eqref{cN-bound-hat} holds for $n=1$. We assume that inequality \eqref{cN-bound-hat} holds for $n=k$ and proceed for the case $n = k + 1$. Using the recursion formula \eqref{Recursion-hat} and the induction hypothesis, for all $\bx\in \bR^d$ we obtain that
$$
\|\hat\cN_{k+1}(\bx)\|\leq\|\sigma(\cP_\mu(\hat\bW_{k+1}\hat\cN_k(\bx))+\hat\bbb_{k+1})-\sigma(\cP_\mu({\bf 0}_{k+1}))\|+\|\sigma(\cP_\mu({\bf 0}_{k+1}))\|.
$$
Using the Lipschitz continuity of $\sigma$ and $\cP_\mu$, we find that
$$
\|\hat\cN_{k+1}(\bx)\|\leq LP\|\hat\bW_{k+1}\|\|\hat\cN_k(\bx)\|+L\|\hat\bbb_{k+1}\|+\|\sigma(\cP_\mu({\bf 0}_{k+1}))\|.
$$
Invoking the induction hypothesis in the right-hand-side of the last inequality yields
\begin{align*}
    \|\hat\cN_{k+1}(\bx)\|
    \leq & LP\|\hat\bW_{k+1}\|\left[\left(\prod_{j=1}^kLP\|
    \hat\bW_j\|\right)\|\bx\|+\sum_{j=1}^k
    \left(\prod_{i=j+1}^kLP\|\hat\bW_i\|\right)(L\|\hat\bbb_j\|+\|\sigma(\cP_\mu({\bf 0}_j))\|)\right]\\
    &+L\|\hat\bbb_{k+1}\|+\|\sigma(\cP_\mu({\bf 0}_{k+1}))\|\\
    =&\left(\prod_{j=1}^{k+1}LP\|\hat\bW_j\|\right)\|\bx\|+\sum_{j=1}^{k+1}
    \left(\prod_{i=j+1}^{k+1}LP\|\hat\bW_i\|\right)(L\|\hat\bbb_j\|+\|\sigma(\cP_\mu({\bf 0}_j))\|).
\end{align*}
Thus,  \eqref{cN-bound-hat} holds for $n=k+1$ and the induction principle ensures that inequality \eqref{cN-bound-hat} holds for all $n\in\bN$.
\end{proof}

Lemma \ref{BOUND-hat} with an additional hypothesis on the norm %which ensures that $\|\sigma({\bf 0}_j)\|<+\infty$ 
guarantees that the sequence $\|\hat\cN_n(\bx)\|$, $n\in\bN$, is uniformly bounded in a bounded set $\DD$. Such a hypothesis will be made clear in later sections in specific contexts.

%By this observation, we introduce the following definition.
%Using the notation defined in Definition \ref{Consecutive-compostion} for consecutive compositions of functions, equations \eqref{neuralnetworks-termk} and \eqref{neuralnetworks-output} may be rewritten as
%$$
%x^{(k)}=\left(\bigodot_{i=1}^k \sigma(\bW_i \cdot+\bbb_i)\right)(x),\ \ 1\le k\le n
%$$
%and
%$$
%y=\bW_{o}\left(\bigodot_{i=1}^n \sigma(\bW_i \cdot+\bbb_i)\right)(x)+b_{o},\ \ x\in[0,1]^d,
%$$
%respectively.

\section{Uniform Convergence of Deep Neural Networks with a Fixed Width}
\setcounter{equation}{0}

In this section, we establish uniform convergence of deep neural networks with a fixed width. 

Throughout this section, we let $l\in \bN$ be fixed and suppose that the sequence of weight matrices $\bW_1\in\bR^{l\times s}$, $\bW_n\in\bR^{l\times l}$, $n\ge 2$, and the sequence of bias vectors  $\bbb_n\in\bR^l$, $n\in\bN$. We then define neural networks $\cN_n(\bx)\in \bR^l$ by \eqref{DNN} {\it without pooling} and they satisfy the recursions \eqref{Step1} and \eqref{Recursion}. We assume that the vector norm used in this section satisfies 
\eqref{vectornorm-zero} and \eqref{nondecreasingvectornorm}.

%The next lemma follows immediately from Lemma \ref{ThreeTerms-hat}.

%\begin{lemma}\label{ThreeTerms} If the activation function $\sigma$ is Lipschitz continuous with a Lipschitz constant $L$, then
%for any positive integers $n$, $m$,
%\begin{align}\label{Basic-Estimate}
%    \|\cN_{n+m}(\bx)-\cN_n(\bx)\|\leq& \sum_{i=0}^{n-1}L^{i+1}\left(\prod_{j=0}^{i-1}\|\bW_{n+m-j}\|\right)\|\bbb_{n+m-i}-\bbb_{n-i}\|\nonumber\\ 
%   & +\sum_{i=0}^{n-2}L^{i+1}\|\cN_{n-1-i}(\bx)\|\left(\prod_{j=0}^{i-1}\|\bW_{n+m-j}\|\right)\|\bW_{n+m-i}-\bW_{n-i}\|\nonumber\\
%   &+L^n\left(\prod_{j=0}^{n-2}\|\bW_{n+m-j}\|\right)\|\bW_{m+1}\cN_m(\bx)-\bW_1\bx\|, \ \ \bx\in \bR^s.
%\end{align}
%\end{lemma}

The next lemma that follows directly from Lemma \ref{BOUND-hat} provides a bound of the deep neural network by the norms of the weight matrices and bias vectors. Here, the widths of the weight matrices and bias vectors are assumed to be variable, which is somewhat more general than what we need in this section. Since the neural networks are constructed without pooling, in this case we have that $\cP_\mu$ is the identity operator and $P=1$.

\begin{lemma}\label{BOUND} If the activation function $\sigma$ is Lipschitz continuous with a Lipschitz constant $L$,
the weight matrices $\bW_n\in\bR^{m_n\times m_{n-1}}$, for $n\in\bN$ with $m_0:=s$, and the bias vectors $\bbb_n\in\bR^{m_n}$ for $n\in\bN$, 
then for all $n\in\bN$
\begin{equation}\label{cN-bound}
    \|\cN_n(\bx)\|\leq \left(\prod_{j=1}^nL\|\bW_j\|\right)\|\bx\|+\sum_{j=1}^n
    \left(\prod_{i=j+1}^nL\|\bW_i\|\right)(L\|\bbb_j\|+\|\sigma({\bf 0}_{m_j})\|),
\end{equation}
where ${\bf 0}_j\in\bR^j$ is a zero vector.
\end{lemma}

Inequality \eqref{cN-bound} in Lemma \ref{BOUND} may be translated to the boundedness of the deep neural network sequence in a bounded domain.

\begin{lemma}\label{Boundedness}
Suppose that $\sigma:\bR\to\bR$ is Lipschitz continuous with a Lipschitz constant $L$, $\DD\subseteq \bR^s$ is bounded by $D>0$, the weight matrices $\bW_n\in\bR^{m_n\times m_{n-1}}$, for $n\in\bN$ with $m_0:=s$, and the bias vectors $\bbb_n\in\bR^{m_n}$ for $n\in\bN$. If there exists a constant $c>0$ such that $\|\bbb_n\|\leq c$, for all $n\in \bN$ and $\bW_n$, $n\in\bN$, satisfy the condition
\begin{equation}\label{BasicAssumption}
    \lim_{n\to\infty}L\|\bW_n\|<1,
\end{equation}
then there exists a positive constant $c_1$ such that  
\begin{equation}\label{BundofNn}
    \sup_{\bx\in\DD}\|\cN_n(\bx)\|\leq c_1D+Lcc_1+c_1\max\{\|\sigma({\bf 0}_{m_j})\|: j\in \bN_n\}, \ \ \mbox{for all}\ \ n\in \bN,
\end{equation}
where ${\bf 0}_j\in\bR^j$ is a zero vector.
If the widths $m_n$ are bounded, then there exists a constant $\rho>0$ such that $\sup_{\bx\in\DD}\|\cN_n(\bx)\|\leq \rho$ for all $n\in \bN$.
\end{lemma}
\begin{proof}
%Because $\DD$ is bounded, there exists a positive constant $c_1$ such that $\|\bx\|\leq c_1$ for all $\bx\in \DD$.
%Since $\bbb_n$, $n\in\bN$, converges, there exists a positive constant $c_2$ such that $\|\bbb_n\|\leq c_2$, for all $n\in \bN$. 
By condition \eqref{BasicAssumption}, according to Lemma \ref{Two_Bounds} with $\alpha_n:=L\|\bW_n\|$, there exists a positive constant $c_1$ such that for all $n\in \bN$,
$$
\prod_{j=1}^{n}L\|\bW_j\|\leq c_1 
$$
and
$$
\sum_{j=1}^{n}
\left(\prod_{i=j+1}^{n}L\|\bW_i\|\right)\leq c_1.
$$
The hypothesis of this lemma ensures that inequality \eqref{cN-bound} in Lemma \ref{BOUND}  holds for all $n\in\bN$. Using inequality \eqref{cN-bound} together with the bounds given above, we obtain the estimate \eqref{BundofNn}.

When the widths $m_n$ are bounded, there is a positive constant $c_2$ such that 
$$
\max\{\|\sigma({\bf 0}_{m_j})\|: j\in \bN_n\}\leq c_2\ \ \mbox{for all}\ \ n\in \bN.
$$
Therefore, \eqref{BundofNn} ensures that
$$
\sup_{\bx\in\DD}\|\cN_n(\bx)\|\leq \rho:=c_1D+Lcc_1+c_1c_2, \ \  \mbox{for all}\ \ n\in \bN,
$$
proving the lemma.
\end{proof}

We are now ready to derive the uniform convergence of deep neural networks with a fixed width from Theorem
\ref{Fixed-widths-hat}.

\begin{theorem}\label{Fixed-widths}
Suppose that $\sigma:\bR\to\bR$ is Lipschitz continuous with a Lipschitz constant $L$ and $\DD\subset\bR^s$ is bounded. 
If the sequences $\bbb_n$, $\bW_n$, $n\in\bN$, converge with \eqref{BasicAssumption},  then the neural networks $\cN_n$ converge uniformly in  $\DD$.
\end{theorem}
\begin{proof}
Note that the neural networks $\cN_n(\bx)$ satisfy the recursions \eqref{Step1} and \eqref{Recursion}. Since the sequence $\bbb_n$, $n\in\bN$, converges, it is bounded. By hypothesis, the sequence $\|\bW_n\|$, $n\in\bN$, satisfies 
\eqref{BasicAssumption}. Moreover, the neural networks have a fixed width. By Lemma \ref{Boundedness}, we know that $\cN_n(\bx)$ is uniformly bounded by a constant $\rho$ for all $n\in\bN$ and $\bx\in\DD$. Therefore, by Theorem \ref{Fixed-widths-hat}, the neural networks $\cN_n(\bx)$, $n\in\bN$, converge uniformly to a function $\cN\in C_l(\DD)$ in $\DD$.
\end{proof}

When we restrict the activation functions to contractions, Theorem \ref{Fixed-widths} is specialized to the uniform convergence theorem established in \cite{HuangXuZhang}, since in the special case of having a contractive activation function the assumption \eqref{BasicAssumption} becomes 
\begin{equation}\label{contraction}
    \lim_{n\to\infty}\|\bW_n\|<\gamma,
\end{equation}
which was assumed in \cite{HuangXuZhang}. When $L=1$, the activation functions are non-expansive. Typical examples of non-expansive activation functions include ReLU and Leaky ReLU. In this case,  the assumption \eqref{BasicAssumption} becomes 
\begin{equation}\label{contraction}
    \lim_{n\to\infty}\|\bW_n\|<1,
\end{equation}
which ensures the uniform convergence of deep neural networks with a non-expansive activation function, according to Theorem \ref{Fixed-widths}. To our best knowledge, uniform convergence of deep neural networks even with non-expansive activation functions, and a fixed width is not available in the literature. Theorem \ref{Fixed-widths} guarantees uniform convergence of neural networks with not only contractive but also  expansive (Lipschitz continuous) activation functions.

The next result on the rate of uniform convergence of neural networks follows directly from Theorem \ref{Order-Fixed-widths-hat}.

\begin{theorem}\label{Order-Fixed-widths}
Suppose that $\sigma:\bR\to\bR$ is Lipschitz continuous with a Lipschitz constant $L$ and $\DD\subseteq \bR^s$ is bounded. If $\bbb_n$ and $\bW_n$ converge to $\bbb^*$ and $\bW^*$ exponentially with \eqref{BasicAssumption}, then the neural networks $\cN_n$ converge to a function $\cN\in C_l(\DD)$ exponentially and uniformly in $\DD$.
\end{theorem}

To close this section, we present a convergence result for deep neural networks with pooling. In building a neural network, pooling is used to enhance features and reduce the dimension of features. 
%Uniform convergence of deep neural networks with the average pooling or the max pooling and with contractive activation functions was established in \cite{HuangXuZhang}. 
Uniform convergence of deep neural networks with the average pooling and a fixed width in the vector norms $\|\cdot\|_p$, for $1\leq p\leq +\infty$, and that of deep neural networks with the max pooling in the  vector norm $\|\cdot\|_\infty$ were established in \cite{HuangXuZhang}.
Theorems \ref{Fixed-widths} and \ref{Order-Fixed-widths} may be modified for uniform convergence of deep neural networks with  pooling. The average pooling and the max pooling are most popular in application. For an integer $\mu\ge 2$, the average pooling $\cP_\mu^A$ is the linear operator from $\bR^{l+\mu}$ to $\bR^l$ defined by
\begin{equation}\label{averagepooling}
	(\cP_\mu^A\bx)_i:=\frac{1}{\mu+1}\sum_{j=0}^\mu x_{i+j},\ \ i\in \bN_l,\ \ \bx\in\bR^{l+\mu}
\end{equation}
and the max pooling $\cP_\mu^M$ is the nonlinear map from $\bR^{l+\mu}$ to $\bR^l$ defined by
\begin{equation}\label{maxpooling}
	(\cP_\mu^M\bx)_i:=\max\{x_{i+j}: j=0, 1,\dots, \mu \},\ \  i\in \bN_l,\ \ \bx\in\bR^{l+\mu}.
\end{equation}
It is known that if $1\le p\le +\infty$, then  for all $\bx\in\bR^{l+\mu}$
\begin{equation}\label{poolingproperty1}
\|\cP_\mu^A \bx\|_p\le \|\bx\|_p, \ \ \mbox{for all}\ \ \bx\in\bR^{l+\mu}, \  p\in [1, +\infty]
\end{equation}
and
\begin{equation}\label{poolingproperty2}
\|\cP_\mu^M \bx-\cP_\mu^M\by\|_p\le (\mu+1)^{1/p}\|\bx-\by\|_p, \ \ \mbox{for all}\ \ \bx,\by\in\bR^{l+\mu}.
\end{equation}
It can be seen from \eqref{poolingproperty1} that $\cP^A_\mu$ is a non-expansive linear operator, and from \eqref{poolingproperty2} that $\cP^M_\mu$ is Lipschitz continuous with the Lipschitz constant $P:=(\mu+1)^{1/p}$ with respect to the vector norm $\|\cdot\|_p$, for $1\leq p\leq \infty$. Clearly, when $1\leq p<\infty$, $\cP^M_\mu$ are expanding and when $p=\infty$, $\cP^M_\mu$ is non-expansive.
In general, by $\cP$ we denote the pooling map from $\bR^{l+\mu}$ to $\bR^l$. Suppose that $\bW_1\in\bR^{(l+\mu)\times s}$, $\bW_n\in\bR^{(l+\mu)\times l}$ for $n\ge 2$, $\bbb_n\in\bR^{l}$ for $n\in\bN$. Deep neural networks with pooling have the form
\begin{equation}\label{cnx2}
	{\cN}^\cP_n(\bx):=\left(\bigodot_{i=1}^n \sigma(\cP(\bW_i \cdot)+\bbb_i)\right)(\bx),\ \ \bx\in\bR^s.
\end{equation}
Clearly, the sequence ${\cN}^\cP_n(\bx)$, $n\in\bN$, satisfies the recursions \eqref{Step1-hat} and \eqref{Recursion-hat}. The next theorem follows directly from Theorems \ref{Fixed-widths-hat} and \ref{Order-Fixed-widths-hat}.

\begin{theorem}\label{Fixed-widths-pooling}
Suppose that $\sigma:\bR\to\bR$ is Lipschitz continuous with the Lipschitz constant $L$, $\cP: \bR^{l+\mu}\to \bR^l$ is a Lipchitz continuous pooling operator with the Lipchitz constant $P$, $\DD\subset \bR^l$ is bounded, and $1\leq p\leq+\infty$. If the sequences $\bbb_n\in \bR^l$, $\bW_n\in\bR^{(l+\mu)\times l}$, $n\in\bN$, converge in the vector norms $\|\cdot\|_p$ and matrix norm $\|\cdot\|_p$, respectively, with 
\begin{equation}\label{LP-condition}
    \lim_{n\to\infty}LP\|\bW_n\|_p<1,
\end{equation}
then the neural networks $\cN_n^\cP$ converge uniformly in $\DD$.

Furthermore, if $\bbb_n$ and $\bW_n$ converge to $\bbb^*$ and $\bW^*$ exponentially, then the neural networks $\cN_n^\cP$ converge to a function $\cN\in C_l(\DD)$ exponentially and uniformly in $\DD$.
\end{theorem}

Note that condition \eqref{LP-condition} for the average pooling and the max pooling is simplified to  
$$
 \lim_{n\to\infty}L\|\bW_n\|_p<1
$$
and
$$
 \lim_{n\to\infty}L(\mu+1)^{1/p}\|\bW_n\|_p<1,
$$
respectively, in light of \eqref{poolingproperty1} and \eqref{poolingproperty2}. Theorem \ref{Fixed-widths-pooling} extends the uniform convergence result with pooling in \cite{HuangXuZhang} from contractive activation functions to Lipschitz continuous activation functions. Moreover, the result for the max pooling in  \cite{HuangXuZhang} is for $p=+\infty$ only. While  Theorem \ref{Fixed-widths-pooling} for the max pooing holds for all norms $\|\cdot\|_p$, $1\leq p\leq+\infty$.

%%%%%%%%%%%%%%%%%%%%%%%%%%%%%%%%%%%%
\section{Uniform Convergence of Deep Neural Networks with Bounded Widths}
%%%%%%%%%%%%%%%%%%%%%%%%%%%%%%%%%%%%

In this section, we extend the uniform convergence theorems established in the last section for deep neural networks with weight matrices of a {\it fixed} width to those with weight matrices of {\it bounded} widths.

%We first describe the setting.
%For each $i\in \bN_n$, we denote by $\bW_i$ and $\bbb_i$ the weight matrix and bias vector, respectively, of the $i$-th hidden layer. That is, $\bbb_i\in\bR^{m_i}$ for $1\le i\le n$, $\bW_1\in \bR^{m_1\times s}$, and $\bW_i\in\bR^{m_i\times m_{i-1}}$ for $2\le i\le n$. The weight matrix $\bW_{o}$ and bias vector $\bbb_{o}$ of the output layer satisfy $\bW_{o}\in \bR^{t\times m_n}$ and $\bbb_{o}\in\bR^t$. 
%We define the deep neural network $\cN_n(\bx)$ by \eqref{DNN}. Clearly,  $\cN_n(\bx)$ is a vector-valued function in $\bR^{m_n}$.

We adopt the setting described in the beginning of section 2 and assume in this section that the matrix widths $m_n$, $n\in \bN$, of the deep neural networks are bounded. Specifically, we assume $l:=\max\{m_n: n\in \bN\}<+\infty$. Thus, $1\leq m_n\leq l$ for all $n\in\bN$ and $m_{n_0}=l$ for some $n_0\in \bN$. We define the deep neural network $\cN_n(\bx)$ by \eqref{DNN}. Clearly,  $\cN_n(\bx)$ is a vector-valued function in $\bR^{m_n}$. The size of $\cN_n(\bx)$ varies according to $n$. The goal of this section is to establish uniform convergence theorems of the sequence $\cN_n(\bx)$, $n\in\bN$. The approach used in the last section for convergence analysis is limited to deep neural networks with weight matrices of a {\it fixed} width. To address this issue, 
we augment the matrix $\bW_n\in \bR^{m_n\times m_{n-1}}$ to $\tilde{\bW}_n\in \bR^{l\times l}$ and the vector $\bbb_n\in\bR^{m_n}$ to $\tilde{\bbb}_n\in \bR^l$ by the zero-padding. That is,
we let
\begin{equation}\label{augmentW1}
\tilde{\bW}_1:=\left[\begin{array}{l}
     \bW_1  \\
     {\bf 0}_{(l-m_1)\times s}
\end{array}\right],
\end{equation}
\begin{equation}\label{augmentWi}
\tilde{\bW}_n:=\left[\begin{array}{ll}
     \bW_n & {\bf 0}_{m_n\times (l-m_{n-1})}  \\
     {\bf 0}_{(l-m_n)\times m_{n-1}}  & {\bf 0}_{(l-m_{n})\times (l-m_{n-1})}
\end{array}\right], \ \ n> 1,
\end{equation}
and 
\begin{equation}\label{Augmentb}
\tilde{\bbb}_n:=\left[\begin{array}{l}
     \bbb_n \\
     {\bf 0}_{l-m_n}
\end{array}\right].
\end{equation}
We then define the deep neural network $\tilde{\cN}_n(\bx)$ by
\begin{equation}\label{tilde-DNN}
    \tilde{\cN}_n(\bx):=\left(\bigodot_{i=1}^n \sigma(\tilde{\bW}_i \cdot+\tilde{\bbb}_i)\right)(\bx),\ \ \bx\in\bR^s.
\end{equation}
Clearly, $\tilde{\cN}_n(\bx)$ is a vector-valued function in $\bR^l$, that is, for all $n\in\bN$,  $\tilde{\cN}_n(\bx)$ have the same size $l$. Moreover,  $\tilde{\cN}_n(\bx)$ has the recursion
\begin{equation}\label{Step2}
    \tilde\cN_1(\bx):=\sigma(\tilde\bW_1 \bx+\tilde\bbb_1)
\end{equation}
and
\begin{equation}\label{Recursion2}
    \tilde{\cN}_{n+1}(\bx)=\sigma(\tilde{\bW}_{n+1}\tilde{\cN}_n(\bx)+\tilde{\bbb}_{n+1}), \ \ \bx\in \bR^s, \ \ \mbox{for all} \ \ n\in \bN.
\end{equation}

We need to define the uniform convergence of  $\cN_n(\bx)$, $n\in \bN$.

\begin{definition}\label{NConv-Ext}
We say that a sequence of vector-valued functions $\cN_n:\bR^s\to \bR^{m_n}$, $n\in \bN$, converges uniformly to a function $\cN\in C_l(\DD)$ if the sequence of the augmented vector-valued functions $\tilde{\cN}_n(\bx)$, $n\in \bN$, converges uniformly to $\cN$ in the space $C_l(\DD)$ in a norm on $\bR^l$. 
\end{definition}

Likewise, we need to define the convergence of  $\bW_n$, $n\in \bN$.

\begin{definition}\label{WConv-Ext}
%We say that a sequence of vectors $\bbb_k\in \bR^{m_k}$, $k\in \bN$ converges as $k\to\infty$ if  $\tilde{\bbb}_k\in \bR^l$, $k\in \bN$, converges as $k\to\infty$, with respect to a norm of $\bR^l$.  
We say that a sequence of matrices $\bW_n\in \bR^{m_n\times m_{n-1}}$, $n\in \bN$ converges to $\bW\in \bR^{l\times l}$ as $n\to\infty$ if  
$$
\lim_{n\to\infty}\|\tilde{\bW}_n-\bW\|=0,
$$
where $\|\cdot\|$ is the induced matrix norm of a vector norm on $\bR^l$. 
\end{definition}

Convergence of a sequence of vectors $\bbb_n\in \bR^{m_n}$, $n\in \bN$, can be understood as convergence of a sequence of $m_n\times 1$ matrices.

We need to understand the relation between the norm of a matrix and that of its augmented matrix with zero blocks.
For $\mu,\nu\in\bN$, suppose that $\bW\in \bR^{\mu\times\nu}$ and for $l\geq \max\{\mu,\nu\}$ we let
\begin{equation}\label{augmentMatrixW}
\tilde{\bW}:=\left[\begin{array}{ll}
     \bW & {\bf 0}_{\mu\times (l-\nu)}  \\
     {\bf 0}_{(l-\mu)\times \nu}  & {\bf 0}_{(l-\mu)\times (l-\nu)}
\end{array}\right].
\end{equation}
In the next lemma, we show that the matrix augmentation process described above preserves certain matrix norms.

\begin{lemma}\label{MatrixNormEqual} For $\mu,\nu\in\bN$, suppose that $l\geq \max\{\mu,\nu\}$.
If $\tilde{\bW}\in \bR^{l\times l}$ is an augmented matrix of $\bW\in \bR^{\mu\times \nu}$ defined by \eqref{augmentMatrixW} and $\|\cdot\|$ is the matrix norm induced from  a vector norm on $\bR^l$ satisfying the extension invariant \eqref{vectornorm-zero} and the monotonicity condition
\eqref{nondecreasingvectornorm},
%one of the $\ell_p$ vector norms, for $1\leq p\leq +\infty$,
then
\begin{equation}\label{matrixNormE}
    \|\tilde{\bW}\|=\|\bW\|.
\end{equation}
\end{lemma}
\begin{proof}
By the definition \eqref{augmentMatrixW}  of $\tilde{\bW}$, we have that
$$
\tilde{\bW}\tilde{\bx}=\left[\begin{array}{c}
    \bW\bx
          \\
        {\bf 0}_{l-\nu}
    \end{array}\right], \ \ \mbox{for all}\ \ \tilde{\bx}:=\left[\begin{array}{c}
    \bx
          \\
         \bx'
    \end{array}\right]\in \bR^l\ \ \mbox{with}\ \ \bx\in\bR^\nu, \ \bx'\in\bR^{l-\nu}.
$$
Since the vector norm $\|\cdot\|$ satisfies \eqref{vectornorm-zero}, 
we obtain that 
\begin{equation}\label{equ-matrixNorm}
    \|\tilde{\bW}\tilde{\bx}\|=\|\bW\bx\|, \ \ \mbox{for all}\ \ \tilde{\bx}:=\left[\begin{array}{c}
    \bx
          \\
         \bx'
    \end{array}\right]\in \bR^l\ \ \mbox{with}\ \ \bx\in\bR^\nu, \ \bx'\in\bR^{l-\nu}.
\end{equation}
The definition of the matrix norm induced from the vector norm together with equation \eqref{equ-matrixNorm} ensures that
\begin{align}\label{matrixNorm}
    \|\tilde{\bW}\|
&=\sup\left\{\frac{\|\bW\bx\|}{\|\tilde{\bx}\|}:\ \ \mbox{for all}\ \ \tilde{\bx}:=\left[\begin{array}{c}
    \bx
          \\
         \bx'
    \end{array}\right]\in \bR^l\ \ \mbox{with}\ \ \bx\in\bR^\nu, \ \bx'\in\bR^{l-\nu}\right\}.
\end{align}
In the right-hand-side of equation \eqref{matrixNorm}, we restrict $\bx'={\bf 0}$ and since the vector norm satisfies the extension invariant property \eqref{vectornorm-zero}, we  note that $\|\tilde{\bx}\|=\|\bx\|$ with this restriction. Therefore, we find that
\begin{equation}\label{inequ-matrixNorm}
    \|\tilde{\bW}\|
\geq \sup\left\{\frac{\|\bW\bx\|}{\|\bx\|}:\bx\in \bR^\nu\right\}=\|\bW\|.
\end{equation}
On the other hand, in light of the monotonicity condition \eqref{nondecreasingvectornorm} that the vector norm satisfies, we observe for all 
$\tilde{\bx}:=\left[\begin{array}{c}
    \bx
          \\
         \bx'
    \end{array}\right]\in \bR^l$ 
with $\bx\in\bR^\nu$ and $\bx'\in\bR^{l-\nu}$ 
that 
$$
\|\tilde{\bx}\|\geq \left\|\left[\begin{array}{c}
    \bx
          \\
         {\bf 0}
    \end{array}\right]\right\| =\|\bx\|,\ \ \mbox{where} \ \ {\bf 0}\in \bR^{l-\nu}.
$$
This inequality together with \eqref{matrixNorm} ensures that
\begin{equation}\label{inequ-matrixNorm2}
    \|\tilde{\bW}\|
\leq \sup\left\{\frac{\|\bW\bx\|}{\|\bx\|}: \bx\in \bR^\nu\right\}=\|\bW\|.
\end{equation}
Combining inequalities \eqref{inequ-matrixNorm} and  \eqref{inequ-matrixNorm2} yields
equation \eqref{matrixNormE}.
\end{proof}

Next, we present the following uniform convergence theorem for deep neural networks with a Lipschitz continuous activation function and with weight matrices of bounded widths.

\begin{theorem}
Suppose that $\sigma:\bR\to\bR$ is Lipschitz continuous with the Lipschitz constant $L$, $\DD\subset \bR^s$ is bounded, $l:=\sup\{m_i: i\in \bN\}<+\infty$ and $\|\cdot\|$ is the matrix norm induced from a vector norm satisfying the extension invariant condition \eqref{vectornorm-zero} and the monotonicity condition 
\eqref{nondecreasingvectornorm}. If the sequences $\bbb_n\in \bR^{m_n}$, $\bW_n\in\bR^{m_n\times m_{n-1}}$, $n\in\bN$, converge, and the sequence $\|\bW_n\|$, $n\in\bN$, satisfies the condition \eqref{BasicAssumption}, then the sequence  $\cN_n$, $n\in\bN$, converges uniformly in  $\DD$.
\end{theorem}
\begin{proof}
According to Definition \ref{NConv-Ext}, it suffices to prove that $\tilde{\cN}_n:\bR^s\to \bR^l$, $n\in\bN$, converges uniformly. This is done by employing Theorem \ref{Fixed-widths}. By the hypothesis of this theorem and Definition \ref{WConv-Ext}, we see that the sequences $\tilde{\bbb}_n\in \bR^l$, $\tilde{\bW}_n\in\bR^{l\times l}$, $n\in\bN$, converge. It remains to prove that  
\begin{equation}\label{Condition}
    \lim_{n\to\infty}L\|\tilde{\bW}_n\|<1.
\end{equation}
Because $\|\cdot\|$ is the matrix norm induced from a vector norm that satisfies the extension invariant condition \eqref{vectornorm-zero} and the monotonicity condition 
\eqref{nondecreasingvectornorm}, by equation \eqref{matrixNormE} of Lemma \ref{MatrixNormEqual}, we conclude that
$$
\|\tilde{\bW}_n\|=\|\bW_n\|, \ \ \mbox{for all}\ \ n\in \bN. 
$$
This equation with the hypothesis that the sequence $\|\bW_n\|$, $n\in\bN$, satisfies the condition \eqref{BasicAssumption} implies that inequality \eqref{Condition} holds. Hence, the conclusion of this theorem follows directly from Theorem \ref{Fixed-widths}.
\end{proof}

We can obtain uniform convergence results for neural networks with Lipschitz continuous activation functions and pooling, and with bounded matrix widths. Moreover, the exponential convergence result for neural networks of this type can also be obtained.

%%%%%%%%%%%%%%%%%%%%%%%%%%%%
\section{Deep Neural Networks with Unbounded Widths}
%%%%%%%%%%%%%%%%%%%%%%%%%%%%

We consider in this section deep neural networks with weight matrices of unbounded widths. In this case, we suppose that $\bW_n\in \bR^{m_n\times m_{n-1}}$, $\bbb_n\in\bR^{m_n}$, $n\in\bN$, and a sequence of deep neural networks $\cN_n(\bx)$, $n\in\bN$, are defined by \eqref{DNN}. Then, the sequence $\cN_n(\bx)$, $n\in\bN$, satisfies \eqref{Recursion}. We further assume that the matrix widths $m_n$, $n\in \bN$, are unbounded. That is, there exists a subsequence $m_{n_i}$, $i\in \bN$, with $\lim_{i\to\infty}m_{n_i}=+\infty$. Due to the unboundedness of the widths, the approach used in the last section is not applicable. We will extend all vectors and matrices to elements in sequence spaces $\ell_p(\bN)$ and $\ell_p(\bN^2)$, respectively, and consider convergence in the sequence spaces.

%We will augment the matrices to semi-infinite matrices. We need the notion of sequence spaces. 
%For $1\leq p\leq +\infty$, by $\ell_p(\bN)$ we denote the space of sequences $\bx$ with $\|\bx\|_{\ell_p(\bN)}<+\infty$, where
%$$
%\|\bx\|_{\ell_p(\bN)}:=\left(\sum_{j=1}^\infty|x_j|^p\right)^\frac{1}{p}, \ \ \mbox{for}\ \ 1\leq p<+\infty
%$$ 
%and
%$$
%\|\bx\|_{\ell_\infty(\bN)}:=\sup\{|x_j|: j\in \bN\}.
%$$ 
%In this section, we reserve $\|\cdot\|_p$ as the vector norm in $\bR^\mu$ for a $\mu<\infty$.
%We also need the notion of spaces of semi-infinite matrices. For $1\leq p\leq +\infty$, we let $\ell_p(\bN^2):=\ell_p(\bN\times\bN)$ denote the spaces of semi-infinite matrices $\bW: \ell_p(\bN)\to\ell_p(\bN)$, viewed as operators, with $\|\bW\|_{\ell_p(\bN^2)}<+\infty$, where $\|\cdot\|_{\ell_p(\bN^2)}$ are operator norms induced from the $\ell_p(\bN)$ norms. 
%We extend the convergence theorem established in the last section for deep neural networks with a fixed width to deep neural networks with variable widths.
%For each $i$ with $1\le i\le n$, let $\bW_i$ and $\bbb_i$ denote respectively the weight matrix and bias vector of the $i$-th hidden layer. That is, $\bbb_i\in\bR^{m_i}$ for $1\le i\le n$, $\bW_1\in \bR^{m_1\times s}$, and $\bW_i\in\bR^{m_i\times m_{i-1}}$ for $2\le i\le n$. The weight matrix $\bW_{o}$ and bias vector $\bbb_{o}$ of the output layer satisfy $\bW_{o}\in \bR^{t\times m_n}$ and $\bbb_{o}\in\bR^t$. 
%We define the deep neural network $\cN_n(\bx)$ by \eqref{DNN}. Clearly,  $\cN_n(\bx)$ is a vector-valued function in $\bR^{m_n}$.

We extend matrices $\bW_n\in \bR^{m_n\times m_{n-1}}$ to $\tilde{\bW}_n\in \ell_p(\bN\times\bN_s)$ for $n=1$ and $\tilde{\bW}_n\in \ell_p(\bN^2)$ for $n>1$, by the zero-padding.
Specifically, we set
\begin{equation}\label{INFaugmentW1}
\tilde{\bW}_1:=\left[\begin{array}{l}
     \bW_1  \\
     {\bf 0}_{\infty\times s}
\end{array}\right]
\end{equation}
where ${\bf 0}_{\infty\times s}\in \ell_p(\bN\times \bN_s)$ is the zero semi-infinite matrix and 
\begin{equation}\label{INFaugmentWi}
\tilde{\bW}_n:=\left[\begin{array}{ll}
     \bW_n & {\bf 0}_{m_n\times \infty}  \\
     {\bf 0}_{\infty\times m_{n-1}}  & {\bf 0}_{\infty\times\infty}
\end{array}\right], \ \ n> 1,
\end{equation}
where ${\bf 0}_{\infty\times m_{n-1}}\in \ell_p(\bN\times \bN_{m_{n-1}})$, ${\bf 0}_{{m_n}\times \infty}\in \ell_p(\bN_{m_{n}}\times \bN)$, and ${\bf 0}_{\infty\times\infty}\in \ell_p(\bN^2)$ are the zero semi-infinite matrices. We extend vectors $\bbb_n\in\bR^{m_n}$ to $\tilde{\bbb}_n\in \ell_p(\bN)$ by
\begin{equation}\label{INFAugmentb}
\tilde{\bbb}_n:=\left[\begin{array}{c}
     \bbb_n \\
     {\bf 0}_\infty
\end{array}\right],
\end{equation}
where ${\bf 0}_\infty\in \ell_p(\bN)$ is the zero vector.
%The size of the zero matrix that appears in \eqref{INFaugmentW1}, \eqref{INFaugmentWi}, or \eqref{INFAugmentb} can be understood from the context.
We define the deep neural network $\tilde{\cN}_n(\bx)$ by
\begin{equation}\label{tilde-DNN}
    \tilde{\cN}_n(\bx):=\left(\bigodot_{i=1}^n \sigma(\tilde{\bW}_i \cdot+\tilde{\bbb}_i)\right)(\bx),\ \ \bx\in\bR^s.
\end{equation}
Clearly, %for a fixed $\bx\in \bR^s$, $\tilde{\cN}_n(\bx)$ is a sequence in $\ell_p(\bN)$ and it has 
we have the recursion
\begin{equation}\label{Unbounded-Step1}
    \tilde\cN_1(\bx):=\sigma(\tilde\bW_1 \bx+\tilde\bbb_1)
\end{equation}
and
\begin{equation}\label{Unbounded-Recursion2}
    \tilde{\cN}_{n+1}(\bx)=\sigma(\tilde{\bW}_{n+1}\tilde{\cN}_n(\bx)+\tilde{\bbb}_{n+1}), \ \ \bx\in \bR^s, \ \ \mbox{for all} \ \ n\in \bN.
\end{equation}

It is important to understand the relation between $\tilde{\cN}_n(\bx)$ and  $\cN_n(\bx)$ for each $n\in\bN$. In this regard, we have the following fact.

\begin{lemma}\label{Extension-to-Inf}
There holds the relation
\begin{equation}\label{Unbounded-tilde_N}
    \tilde{\cN}_n(\bx)=\left[\begin{array}{c}
         \cN_n(\bx) \\
         \sigma({\bf 0}_\infty)
    \end{array} \right].
\end{equation}
%where $\sigma_\infty({\bf 0}):=\sigma({\bf 0})$ with ${\bf 0}$ being a vector in $\ell_p(\bN)$.
\end{lemma}
\begin{proof}
We prove equation \eqref{Unbounded-tilde_N} by induction on $n$. When $n=1$, by \eqref{INFaugmentW1} and
by the definition of $\cN_1(\bx)$ and $\tilde{\cN}_1(\bx)$, we observe for $\bx\in \bR^s$ that
$$
\tilde{\cN}_1(\bx)=\sigma\left(\left[\begin{array}{l}
    \bW_1\bx+  \bbb_1 \\
    {\bf 0}_\infty
\end{array}\right] \right)=\left[\begin{array}{l}
    \sigma( \bW_1\bx +\bbb_1) \\
    \sigma({\bf 0}_\infty)
\end{array} \right]=\left[\begin{array}{l}
         \cN_1(\bx) \\
         \sigma({\bf 0}_\infty)
    \end{array} \right].
$$
That is, equation \eqref{Unbounded-tilde_N} holds for $n=1$. 

We assume that equation  \eqref{Unbounded-tilde_N}  holds for $n=k$ and proceed for the case $n=k+1$. By recursion \eqref{Unbounded-Recursion2}, the induction hypothesis,  the definition \eqref{INFaugmentWi} of matrix $\tilde{\bW}_{k+1}$, the definition \eqref{INFAugmentb} of vector $\tilde{\bbb}_{k+1}$, and the recursion \eqref{Recursion}, we obtain that 
\begin{align*}
    \tilde{\cN}_{k+1}(\bx)&=\sigma\left(\tilde{\bW}_{k+1}\left[\begin{array}{c}
         \cN_k(\bx) \\
         {\bf 0}_\infty
    \end{array} \right]+\tilde{\bbb}_{k+1}\right)\\
    &=\sigma\left(\left[\begin{array}{l}
    \bW_{k+1}\cN_k(\bx) \\
    {\bf 0}_\infty
\end{array}\right]+\left[\begin{array}{l}
    \bbb_{k+1} \\
    {\bf 0}_\infty
\end{array}\right]\right)\\
&=\left[\begin{array}{l}
    \sigma(\bW_{k+1}\cN_k(\bx)+ \bbb_{k+1})  \\
    \sigma({\bf 0}_\infty)
\end{array}\right]\\
&=\left[\begin{array}{l}
    \cN_{k+1}(\bx)  \\
    \sigma({\bf 0}_\infty)
\end{array}\right].
\end{align*}
Thus, equation \eqref{Unbounded-tilde_N} holds for $n=k+1$. The induction principle ensures that equation \eqref{Unbounded-tilde_N}  holds for all $n\in\bN$.
\end{proof}

Lemma \ref{Extension-to-Inf} reveals that when $\sigma(0)=0$, which is satisfied by many activation functions such as ReLU, hyperbolic tangent, the Gaussian error linear unit, the exponential linear unit, the scaled exponential linear unit and the sigmoid linear unit, we have that 
\begin{equation}\label{Unbounded-tilde_N-zero}
    \tilde{\cN}_n(\bx)=\left[\begin{array}{l}
         \cN_n(\bx) \\
         {\bf 0}_\infty
    \end{array} \right],
\end{equation}
which is in $\ell_p(\bN)$ for all $1\leq p\leq+\infty$. However, when  $\sigma(0)\neq 0$, which includes sigmoid, softplus, and Gaussian, $\tilde{\cN}_n(\bx)\notin \ell_p(\bN)$ for any $p\in [1,+\infty)$, but  $\tilde{\cN}_n(\bx)\in \ell_\infty(\bN)$. 

We now define the notion of uniform convergence of  $\cN_n\in \bR^{m_n}$, $n\in \bN$.

\begin{definition}\label{NConv-Inf} Let $\DD\subset\bR^s$ be bounded.
We say that a sequence of vector-valued functions $\cN_n:\bR^s\to \bR^{m_n}$, $n\in \bN$, converges uniformly in $\DD$ to a function $\cN:\bR^s\to \ell_p(\bN)$, for $1\leq p\leq \infty$, if the sequence of the extended functions $\tilde{\cN}_n:\bR^s\to \ell_p(\bN)$, $n\in \bN$, converges uniformly to $\cN\in C_{\ell_p(\bN)}(\DD)$. 
\end{definition}

We need to define the convergence of  $\bW_n\in \bR^{m_n\times m_{n-1}}$, $n\in \bN$.

\begin{definition}\label{WConv-Inf}
%For a $\bW\in \ell_p(\bN^2)$, let
%$$
%\|\bW\|_{\ell_p(\bN^2)}:=\sup\left\{\frac{\|\bW \cN\|_{\ell_p(\bN)}}{\|\cN\|_{\ell_p(\bN)}}: \cN\in \ell_p(\bN), \cN\neq {\bf 0}_\infty\right\}.
%$$
We say that a sequence of matrices $\bW_n\in \bR^{m_n\times m_{n-1}}$, $n\in \bN$, converges to $\bW\in \ell_p(\bN^2)$ as $n\to\infty$ if  
$$
\lim_{n\to\infty}\|\tilde{\bW}_n-\bW\|_{\ell_p(\bN^2)}=0,
$$
where $\tilde\bW_n$ is the extension of $\bW_n$ by \eqref{INFaugmentW1} and \eqref{INFaugmentWi}.
\end{definition}

Likewise, we need to define the convergence of  $\bbb_n\in \bR^{m_n}$, $n\in \bN$.

\begin{definition}\label{bConv-Inf}
We say that a sequence of vectors $\bbb_n\in \bR^{m_n}$, $n\in \bN$, converges to $\bbb\in \ell_p(\bN)$ as $n\to\infty$ if  
$$
\lim_{n\to\infty}\|\tilde{\bbb}_n-\bbb\|_{\ell_p(\bN)}=0,
$$
where $\tilde{\bbb}_n$ is the extension of $\bbb_n$ by \eqref{INFAugmentb}.
\end{definition}

%The next result follows from Lemma \ref{ThreeTerms-hat}.

%\begin{lemma}\label{tilde-ThreeTerms} If the activation function $\sigma$ is Lipschitz continuous with a Lipschitz constant $L$, then
%for any positive integers $n$, $m$,
%\begin{align}\label{tilde-Basic-Estimate}
%    \|\tilde{\cN}_{n+m}(\bx)&-\tilde{\cN}_n(\bx)\|_{\ell_p(\bN)}\leq \sum_{i=0}^{n-1}L^{i+1}\left(\prod_{j=0}^{i-1}\|\tilde{\bW}_{n+m-j}\|_{\ell_p(\bN^2)}\right)\|\tilde{\bbb}_{n+m-i}-\tilde{\bbb}_{n-i}\|_{\ell_p(\bN)}\nonumber\\ 
%   & +\sum_{i=0}^{n-2}L^{i+1}\|\tilde{\cN}_{n-1-i}(\bx)\|_{\ell_p(\bN)}\left(\prod_{j=0}^{i-1}\|\tilde{\bW}_{n+m-j}\|_{\ell_p(\bN^2)}\right)\|\tilde{\bW}_{n+m-i}-\tilde{\bW}_{n-i}\|_{\ell_p(\bN^2)}\nonumber\\
%   &+L^n\left(\prod_{j=0}^{n-2}\|\tilde{\bW}_{n+m-j}\|_{\ell_p(\bN^2)}\right)\|\tilde{\bW}_{m+1}\tilde{\cN}_m(\bx)-\tilde{\bW}_1\bx\|_{\ell_p(\bN)}, \ \ \bx\in \bR^s.
%\end{align}
%\end{lemma}
%\begin{proof}
%It suffices to show that $\|\tilde{\cN}_n(\bx)-\tilde{\cN}_m(\bx)\|_{\ell_p(\bN)}$ is bounded by the right-hand-side of \eqref{tilde-Basic-Estimate}.
%Since the recursion \eqref{Unbounded-Recursion2} of $\tilde{\cN}_n(\bx)$
%and the recursion \eqref{Recursion} of $\cN_n(\bx)$ have exactly the same structure, the proof of Lemma \ref{ThreeTerms} may be modified by replacing $\cN_n$, $\bW_n$ and $\bbb_n$ by $\tilde{\cN}_n$, $\tilde{\bW}_n$ and $\tilde{\bbb}_n$, respectively, to establish the inequality. We leave the details to the interested readers.
%\end{proof}

The next lemma concerns the uniform boundedness of the sequence of the deep neural networks $\tilde\cN_n(\bx)$, $n\in\bN$.

\begin{lemma}\label{Ext-Boundedness}
Suppose that $\sigma:\bR\to\bR$ is Lipschitz continuous with a Lipschitz constant $L$, the weight matrices $\bW_n\in\bR^{m_n\times m_{n-1}}$, for $n\in\bN_n$ with $m_0:=s$, and the bias vectors $\bbb_n\in\bR^{m_n}$ for $n\in\bN_n$. Suppose that $p\in [1,+\infty]$ if  $\sigma(0)=0$ and $p=+\infty$ if $\sigma(0)\neq 0$.
If there exists a constant $c>0$ such that $\|\bbb_n\|_p\leq c$, for all $n\in \bN$ and the sequence $\|\bW_n\|$, $n\in\bN$, satisfies the condition \eqref{BasicAssumption}
%\begin{equation}\label{BasicAssumption}
 %   \lim_{n\to\infty}L\|\bW_n\|_p<1
%\end{equation}
and $\DD\subset \bR^s$ is bounded, then there exists a positive constant $\rho$ such that  
$$
\|\tilde{\cN}_n(\bx)\|_{\ell_p(\bN)}\leq \rho,%+\|\sigma({\bf 0}_n)\|_p, 
\ \ \mbox{for all}\ \ n\in \bN, \ \bx\in \DD.
$$ 
%where $\rho$ is defined in Lemma \ref{Boundedness}.
\end{lemma}
\begin{proof} 
When $\sigma(0)=0$, by equation \eqref{Unbounded-tilde_N} of Lemma \ref{Extension-to-Inf}, we obtain that
\begin{equation}\label{EqualNorm}
    \|\tilde{\cN}_n(\bx)\|_{\ell_p(\bN)}=\|\cN_n(\bx)\|_p,
\end{equation}
for all $p\in [1, +\infty]$. In this case, estimate \eqref{BundofNn} of Lemma \ref{Boundedness} reduces to  
$$
\sup_{\bx\in\DD}\|\cN_n(\bx)\|_p\leq c_1D+Lcc_1,\ \ \mbox{for all}\ \ n\in\bN,
$$ 
where $D>0$ is an upper bound of $\DD$.
This estimate combined with \eqref{EqualNorm} ensures that
$\|\tilde{\cN}_n(\bx)\|_{\ell_p(\bN)}$ are uniformly bounded.

When  $\sigma(0)\neq 0$, by noticing that $\|\sigma({\bf 0}_\infty)\|_{\ell_\infty(\bN)}=|\sigma(0)|$ and again by equation \eqref{Unbounded-tilde_N}, we find that
\begin{equation}\label{NormBound}
    \|\tilde{\cN}_n(\bx)\|_{\ell_\infty(\bN)}
=\max\{\|\cN_n(\bx)\|_\infty, |\sigma(0)|\}.
\end{equation}
Estimate \eqref{BundofNn} ensures that 
$$
\sup_{\bx\in\DD}\|\cN_n(\bx)\|_\infty\leq c_1D+Lcc_2+c_1|\sigma(0)|, \ \ \mbox{for all}\ \ n\in\bN.
$$
%\textcolor{red}{(Note: The RHS above should be $c_1D+Lcc_1+c_1|\sigma(0)|$.)}

Substituting this bound into the right-hand-side of equation \eqref{NormBound} leads to the boundedness of 
$\sup_{\bx\in\DD} \|\tilde{\cN}_n(\bx)\|_{\ell_\infty(\bN)}$.
\end{proof}

We need to understand the relation between the norm of a matrix and that of its extension by the zero-padding. To this end, 
for $\mu,\nu\in\bN$, we suppose that $\bW\in \bR^{\mu\times\nu}$and we define the extension by
\begin{equation}\label{augmentMatrixW-Inf}
\tilde{\bW}:=\left[\begin{array}{ll}
     \bW & {\bf 0}_{\mu\times \infty}  \\
     {\bf 0}_{\infty\times \nu}  & {\bf 0}_{\infty\times \infty}
\end{array}\right].
\end{equation}
In the next lemma, we show that the extension \eqref{augmentMatrixW-Inf} preserves matrix norms.

\begin{lemma}\label{MatrixNormEqual-Inf} 
Let $p\in [1, +\infty]$. If $\tilde{\bW}\in \ell_p(\bN^2)$ is the extension of matrix $\bW\in \bR^{\mu\times \nu}$ defined by equation \eqref{augmentMatrixW-Inf},
then
\begin{equation}\label{matrixNormE-Inf}
    \|\tilde{\bW}\|_{\ell_p(\bN^2)}=\|\bW\|_p.
\end{equation}
\end{lemma}
\begin{proof}
By  definition \eqref{augmentMatrixW-Inf}  of the extension $\tilde{\bW}$, we have that
$$
\tilde{\bW}\tilde{\bx}=\left[\begin{array}{c}
    \bW\bx
          \\
        {\bf 0}_\infty
    \end{array}\right], \ \ \mbox{for all}\ \ \tilde{\bx}:=\left[\begin{array}{c}
    \bx
          \\
         \bx'
    \end{array}\right]\in \ell_p(\bN)\ \ \mbox{with}\ \ \bx\in\bR^\nu, \ \bx'\in\ell_p(\bN).
$$
We obtain that 
\begin{equation}\label{equ-matrixNorm-Inf}
    \|\tilde{\bW}\tilde{\bx}\|_{\ell_p(\bN)}=\|\bW\bx\|_p, \ \ \mbox{for all}\ \ \tilde{\bx}:=\left[\begin{array}{c}
    \bx
          \\
         \bx'
    \end{array}\right]\in \ell_p(\bN)\ \ \mbox{with}\ \ \bx\in\bR^\nu, \ \bx'\in\ell_p(\bN).
\end{equation}
The definition of the norm $\|\tilde\bW\|_{\ell_p(\bN^2)}$ together with \eqref{equ-matrixNorm-Inf} ensures that
\begin{align}\label{matrixNorm-Inf}
    \|\tilde{\bW}\|_{\ell_p(\bN^2)}
&=\sup\left\{\frac{\|\bW\bx\|_p}{\|\tilde{\bx}\|_{\ell_p(\bN)}}:\ \ \mbox{for all}\ \ \tilde{\bx}:=\left[\begin{array}{c}
    \bx
          \\
         \bx'
    \end{array}\right]\in \ell_p(\bN)\ \ \mbox{with}\ \ \bx\in\bR^\nu, \ \bx'\in\ell_p(\bN)\right\}.
\end{align}
In the right-hand-side of equation \eqref{matrixNorm-Inf}, we restrict $\bx'={\bf 0}$ and note that $\|\tilde{\bx}\|_{\ell_p(\bN)}=\|\bx\|_p$ with this restriction. Thus, we find that
\begin{equation}\label{inequ-matrixNorm-Inf}
     \|\tilde{\bW}\|_{\ell_p(\bN^2)}
\geq \sup\left\{\frac{\|\bW\bx\|_p}{\|\bx\|_p}:\bx\in \bR^\nu\right\}=\|\bW\|_p.
\end{equation}
On the other hand, we observe for all 
$ \tilde{\bx}:=\left[\begin{array}{l}
    \bx
          \\
         \bx'
    \end{array}\right]\in \ell_p(\bN)$ with $\bx\in\bR^\nu, \ \bx'\in \ell_p(\bN)$ that 
$$
\|\tilde{\bx}\|_{\ell_p(\bN)}\geq \left\|\left[\begin{array}{l}
    \bx
          \\
         {\bf 0}_\infty
    \end{array}\right]\right\|_{\ell_p(\bN)} =\|\bx\|_p.
$$
This inequality together with \eqref{matrixNorm-Inf} ensures that
\begin{equation}\label{inequ-matrixNorm2-Inf}
    \|\tilde{\bW}\|_{\ell_p(\bN^2)}
\leq \sup\left\{\frac{\|\bW\bx\|_p}{\|\bx\|_p}: \bx\in \bR^\nu\right\}=\|\bW\|_p.
\end{equation}
Combining inequalities \eqref{inequ-matrixNorm-Inf} and  \eqref{inequ-matrixNorm2-Inf} yields
equation \eqref{matrixNormE-Inf}.
\end{proof}

With the help of Lemmas \ref{Ext-Boundedness} and \ref{MatrixNormEqual-Inf}, we have the following uniform convergence result for neural networks with weight matrices of unbounded widths.

\begin{theorem}\label{UniformConvInf}
Suppose that $\sigma:\bR\to\bR$ is Lipschitz continuous with the Lipschitz constant $L$,  and $p\in [1,+\infty]$ if  $\sigma(0)=0$ and $p=+\infty$ if $\sigma(0)\neq 0$. If the sequences $\bbb_n\in \bR^{m_n}$, $n\in\bN$, and $\bW_n\in\bR^{m_n\times m_{n-1}}$, $n\in\bN$, converge,
the sequence $\|\bW_n\|$, $n\in\bN$, satisfies the condition \eqref{BasicAssumption},
and $\DD\subset \bR^s$ is bounded, 
then the sequence of neural networks $\cN_n$ converges uniformly in  $\DD$ to a function in $C_{\ell_p(\bN)}(\DD)$.

In addition, if $\bbb_n$ and $\bW_n$ converge to $\bbb^*$ and $\bW^*$ exponentially, then  the sequence of neural networks $\cN_n$ converges to a function $\cN\in C_{\ell_p(\bN)}(\DD)$ exponentially and uniformly in $\DD$.
\end{theorem}
\begin{proof}
By Definition \ref{NConv-Inf}, it suffices to prove that the sequence $\tilde{\cN}_n(\bx)$, $n\in \bN$, converges uniformly in $C_{\ell_p(\bN)}(\DD)$.  When $p$ satisfies the condition of this theorem according to $\sigma(0)$, by Lemma \ref{Ext-Boundedness}, $\sup_{\bx\in\DD}\|\tilde\cN_n(\bx)\|_{\ell_p(\bN)}$ is bounded for all $n\in\bN$. Because the sequence $\|\bW_n\|$, $n\in\bN$, satisfies the condition \eqref{BasicAssumption},
according to Lemma \ref{MatrixNormEqual-Inf},  we have that
$$
\lim_{n\to\infty}L\|\tilde\bW_n\|_{\ell_p(\bN)}<1.
$$
Thus, Theorems \ref{Fixed-widths-hat} with $\cP_\mu$ being the identity operator ensures that the sequence 
$\tilde\cN_n(\bx)$, $n\in\bN$, converges uniformly in  $C_{\ell_p(\bN)}(\DD)$.

The second part of this theorem follows from Theorem  \ref{Order-Fixed-widths-hat} with $\cP_\mu$ being the identity operator.
%
%The results of this theorem follows from Theorems \ref{Fixed-widths-hat} and \ref{Order-Fixed-widths-hat} with $\cP_\mu$ being the identity operator.
%The rest of the proof may be done by modifying the proof of Theorem \ref{Fixed-widths} with replacing $\cN_n$, $\bW_n$ and $\bbb_n$ by  $\tilde\cN_n$, $\tilde\bW_n$ and $\tilde\bbb_n$, respectively, and with employing Lemmas \ref{tilde-ThreeTerms} and \ref{Ext-Boundedness}. 
\end{proof}

Theorem \ref{UniformConvInf} may be extended to neural networks with Lipschitz continuous poolings and unbounded matrix widths.

%Similarly to Theorem \ref{Order-Fixed-widths}, we have the following result regarding the convergence order.

%\begin{theorem}\label{Order-Fixed-widths-Inf}
%Suppose that $\sigma:\bR\to\bR$ is Lipschitz continuous with a Lipschitz constant $L$, $\DD$ is a bounded set in $\bR^s$, and $p\in [1,+\infty]$ if  $\sigma(0)=0$ and $p=+\infty$ if $\sigma(0)\neq 0$.  If $\bbb_n$ and $\bW_n$ converge to $\bbb^*$ and $\bW^*$ exponentially with \eqref{BasicAssumption}, then the neural networks $\cN_n$ converge to a function $\cN\in C_{\ell_p(\bN)}(\DD)$ exponentially and uniformly in $\DD$.
%\end{theorem}

\section{Uniform Convergence of Convolutional Neural Networks}

In this section, we establish uniform convergence results of  convolutional neural networks (CNNs). Pointwise convergence of CNNs with the ReLU activation function was investigated in \cite{XuZhang2022} by considering CNNs as deep neural networks with weight matrices of increasing widths.
We study uniform convergence of CNNs with Lipschitz continuous activation functions by considering two types of matrix extensions.
We first consider the matrix extension by the zero-padding which is described in the last section, and apply Theorem \ref{UniformConvInf} to CNNs with Lipschitz continuous activation functions to obtain a uniform convergence result of CNNs. We then consider the extension of the weight matrices to semi-infinite matrices by the constant-padding along the diagonals, which is natural for CNNs, and present a uniform convergence theorem for CNNs, with a weaker hypothesis on the filter mask.

We now recall the construction of CNNs.
Given a vector $\bx:=[x_1, x_2,\dots,x_l]^\top\in\bR^l$ and a filter mask $\bw:=[w_0,w_1,\dots,w_\tau]^\top$, the convolution $\bx*\bw$ of $\bx$ with $\bw$ is a vector in $\bR^{l+\tau}$ defined by
$$
(\bx*\bw)_i:=\sum_{j=\max(0,i-l)}^{\min(i-1,\tau)}\bw_j\bx_{i-j},\ \  i\in \bN_{l+\tau}.
$$
For $n\in \bN$, by $\tau_n$ we denote a sequence of positive integers.
For each $n\in \bN$, given a filter mask $\bw^{(n)}:=[w^{(n)}_0, w^{(n)}_1, \dots, w^{(n)}_{\tau_n}]^\top\in\bR^{\tau_n+1}$ and a bias vector $\bbb_n\in \bR^{m_n}$, we construct a CNN by
\begin{equation}\label{CNN-Convolution}
    \cN_n(\bx):=\sigma(\cN_{n-1}(\bx)*\bw^{(n)}+\bbb_n),
\end{equation}
with $\cN_0(\bx):=\bx$. In \eqref{CNN-Convolution}, $\sigma: \bR\to \bR$ is a Lipschitz continuous activation function with the Lipschitz constant $L$.

One can express the convolution $\bx*\bw$ via multiplication of $\bx$ with an $(l+\tau)\times l$ Toeplitz matrix 
$$
T_{ij}:=\left\{
\begin{aligned}
w_{i-j},&\quad 1\leq j\leq i\leq j+\tau\leq l+\tau,\\
0,&\quad \mbox{otherwise},
\end{aligned}
\right.
$$
in the form
$$
\bx*\bw=T\bx.
$$
Clearly, matrix $T$ has a specific form
$$
T=\left[
\begin{array}{ccccccc}
     w_0&0&0&0&\cdots &\cdots &0 \\
     w_1&w_0&0&0&\cdots&\cdots &0\\
     \vdots&\ddots&\ddots&\ddots&\ddots&\ddots&\vdots\\
     w_\tau&w_{\tau-1}&\cdots&w_0&0&\cdots&0\\
     0&w_\tau&\cdots&w_1&w_0&0\cdots&0\\
     \vdots&\ddots&\ddots&\ddots&\ddots&\ddots&\vdots\\
     \cdots&\cdots&0&w_\tau&w_{\tau-1}&\cdots&w_0\\
     \cdots&\cdots&\cdots&0&w_\tau\cdots&\cdots&w_1\\
     \vdots&\ddots&\ddots&\ddots&\ddots&\ddots&\vdots\\
     0&\cdots&\cdots&\cdots0& \cdots&w_\tau&w_{\tau-1}\\
     0&\cdots&\cdots&\cdots&\cdots&\cdots0&w_\tau
\end{array}
\right].
$$
That is, we map the filter mask $\bw$ to the Toeplitz matrix $T$ whose diagonal and the first to the $\tau$th sub-diagonals are the components of $\bw$.
The CNNs $\cN_n$ may be written in a form of the DNN with increasing widths:
\begin{equation}\label{cnnwidth}
    m_0:=s, \ \ m_n:=m_{n-1}+\tau_n,\ \ n\in\bN.
\end{equation}
Introducing weight matrices $\bW_n\in\bR^{m_n\times m_{n-1}}$ defined by
\begin{equation}\label{cnnweightmatrix}
(\bW_n)_{ij}:=\left\{
\begin{aligned}
w^{(n)}_{i-j},&\quad 1\leq j\leq i\leq j+\tau_n\leq m_{n-1}+\tau_n,\\
0,&\quad \mbox{otherwise},
\end{aligned}
\right.
\end{equation}
we have for $n\in \bN$ that
\begin{equation}\label{CNN-Convolution2}
    \cN_n(\bx)=\sigma(\bW_n\cN_{n-1}(\bx)+\bbb_n), \ \ \bx\in \bR^s.
\end{equation}
Clearly, from \eqref{cnnwidth}, the CNNs $\cN_n$ are deep neural networks with unbounded widths.

We first consider uniform convergence of CNNs $\cN_n$ in the sense of Definition \ref{NConv-Inf}.
To this end, we extend $\bW_1$, $\bW_n$, $n\geq2$, and $\bbb_n$, $n\in\bN$, to $\tilde{\bW}_1\in\ell_p(\bN\times \bN_s)$, $\tilde{\bW}_n\in\ell_p(\bN\times\bN)$, $n\geq2$,  and $\tilde{\bbb}_n\in\ell_p(\bN)$, $n\in\bN$, respectively,
by \eqref{INFaugmentW1}, \eqref{INFaugmentWi} and \eqref{INFAugmentb}. We then define $\tilde{\cN}_n(\bx)$ by
\eqref{Unbounded-Recursion2} and note that $\tilde{\cN}_n(\bx)$ is an extension of $\cN_n(\bx)$, that results from the zero-padding. We confine ourselves to filter masks of a fixed length.  Uniform convergence of CNNs in this sense follows directly from Theorem \ref{UniformConvInf}.

We need the following preliminary result.

\begin{lemma}\label{Condition-on-Mask} Suppose that $\tau$ is a positive integer and $\tau_n=\tau$, for all $n\in \bN$.
If $w_i^{(n)}$  satisfies the condition
\begin{equation}\label{Mask-condition}
    w_i^{(n)}= o\left(1\right), \ \  \mbox{for}\ \ i=0,1,\dots,\tau,
\end{equation}
then there hold

(i) $\tilde{\bW}_n$, $n\in\bN$, is a Cauchy sequence;

(ii) $\lim_{n\to+\infty}L\|\bW_n\|_p<1$ for all $p\in [1, +\infty]$, where $L$ is the Lipschitz constant of $\sigma$.
\end{lemma}
\begin{proof}
Suppose that $n$ and $m$ are arbitrary positive integers with $n>1$. A direct computation leads to 
\begin{equation}\label{WCauchy1}
    \|\tilde{\bW}_{n+m}-\tilde{\bW}_n\|_{\ell_1(\bN)}\leq \sum_{k=0}^\tau|w_k^{(n+m)}-w_k^{(n)}|+\sum_{k=0}^\tau|w_k^{(n+m)}|
\end{equation}
and
\begin{equation}\label{WCauchyInf}
\|\tilde{\bW}_{n+m}-\tilde{\bW}_n\|_{\ell_\infty(\bN)}\leq \sum_{k=0}^\tau|w_k^{(n+m)}-w_k^{(n)}|+\sum_{k=0}^\tau|w_k^{(n+m)}|.
\end{equation}
By the interpolation theorem of Mitjagin \cite{Mit}, it follows from inequalities \eqref{WCauchy1} and \eqref{WCauchyInf} for all $p\in [1,\infty]$ that
\begin{equation}\label{WCauchyInf_p}
\|\tilde{\bW}_{n+m}-\tilde{\bW}_n\|_{\ell_p(\bN)}\leq \sum_{k=0}^\tau|w_k^{(n+m)}-w_k^{(n)}|+\sum_{k=0}^\tau|w_k^{(n+m)}|.
\end{equation}
%\textcolor{red}{(Note: I estimate that the RHS of (7.6)-(7.8) should be
%$$
%\leq\sum_{k=0}^\tau|w_k^{(n+m)}-w_k^{(n)}|+\sum_{k=0}^\tau|w_k^{(n+m)}|.
%$$
%There is no $m\tau$ as otherwise it will not be arbitrarily small for all $m$.)
%} 
Let $\epsilon>0$ be arbitrarily small. According to hypothesis \eqref{Mask-condition}, there exists $N\in \bN$ such that for all $n, m>N$, the right-hand-side of \eqref{WCauchyInf_p} is bounded by $\epsilon$, which ensures that the sequence $\tilde\bW_n$, $n\in \bN$, is Cauchy in spaces $\ell_p(\bN)$, for $p\in[1, +\infty]$.

It remains to prove (ii). According to the Riesz-Thorin interpolation theorem, we observe for all $p\in [1,+\infty]$ that 
$$
\|\bW_n\|_p\leq \sum_{k=0}^\tau|w_k^{(n)}|.
$$
Invoking hypothesis \eqref{Mask-condition} in the right-hand-side of the above estimate, we conclude for all $p\in [1,+\infty]$ that for $0<\epsilon_0<1$, there exists $N\in\bN$ with $L\leq N$ such that for all $n>N$
$$
L\|\bW_n\|_p\leq L\sum_{k=0}^\tau|w_k^{(n)}|<\epsilon_0<1,
$$
%\textcolor{red}{(Note: Since $L$ is a constant and $\tau$ is a constant as well, (7.5) can be relaxed to
%$$
%\omega_i^{(n)}=o(1),\ 0\le i\le \tau.
%$$
%Then the equation above can be changed to
%$$
%L\|\bW_n\|_p\leq L\sum_{k=0}^\tau|w_k^{(n)}|<\epsilon_0<1,
%$$)
%}
%
which implies (ii).
\end{proof}

Lemma \ref{Condition-on-Mask} enables us to derive the next theorem.

\begin{theorem}\label{UniformConvInfCNN}
Suppose that $\sigma:\bR\to\bR$ is Lipschitz continuous with a Lipschitz constant $L$,  $\tau_n=\tau$, for all $n\in \bN$, with $\tau$ being a positive integer, and $\DD$ is a bounded set in $\bR^s$. If the sequences $\bbb_n\in \bR^{m_n}$, $n\in\bN$, converges and $w_i^{(n)}$  satisfies the condition
\eqref{Mask-condition}, then the CNNs $\cN_n$ converge uniformly on  $\DD$ to a function in $C_{\ell_p(\bN)}(\DD)$, for $p\in [1,+\infty]$ if  $\sigma(0)=0$ and $p=+\infty$ if $\sigma(0)\neq 0$.
\end{theorem}
\begin{proof}
The hypothesis of this theorem together with Lemma \ref{Condition-on-Mask} ensures that the assumption of 
Theorem \ref{UniformConvInf} is satisfied. Hence, the result of this theorem follows directly from Theorem \ref{UniformConvInf}.
\end{proof}

The exponential convergence of the CNNs that we present next follows directly from the second part of Theorem \ref{UniformConvInf}.

\begin{theorem}\label{UniformConvRateInfCNN}
Suppose that $\sigma:\bR\to\bR$ is Lipschitz continuous with a Lipschitz constant $L$,  $\tau_n=\tau$, for all $n\in \bN$, with $\tau$ being a positive integer, and $\DD$ is a bounded set in $\bR^s$.  If the sequences $\bbb_n\in \bR^{m_n}$, $n\in\bN$, converges exponentially and $w_i^{(n)}$  satisfies the condition
\begin{equation}\label{Mask-condition-Exp}
    w_i^{(n)}={\cal O}\left(r^n\right), \ \  \mbox{for some }\ \ 0<r<1 \mbox{ and for all}\ \ i=0,1,\dots,\tau,
\end{equation}
then the neural networks $\cN_n$ converge uniformly and exponentially on  $\DD$ to a function in $C_{\ell_p(\bN)}(\DD)$, for $p\in [1,+\infty]$ if  $\sigma(0)=0$ and $p=+\infty$ if $\sigma(0)\neq 0$.
\end{theorem}
\begin{proof}
Condition \eqref{Mask-condition-Exp} implies that hypothesis \eqref{Mask-condition} is satisfied, which in turn ensures that Lemma \ref{Condition-on-Mask} holds true. Moreover, it implies that the sequence $\bW_n$, $n\in \bN$, converges exponentially. The result of this theorem follows directly from the second part of Theorem \ref{UniformConvInf}.
\end{proof}

Hypothesis \eqref{Mask-condition} on the filter mask may be relaxed if a different extension of matrices $\bW_n$, $n\ge 2$, is adopted. 
We next present a relaxation on the hypothesis \eqref{Mask-condition} of the filter mask. 
To this end, instead of using the zero-padding adopted previously for matrix extension, we adopt the constant padding along each of the diagonal and sub-diagonals.
Specifically, we extend  $\bW_n$, $n\geq2$, to  $\bar{\bW}_n\in\ell_p(\bN^2)$
by 
\begin{equation}\label{ConstantPadding}
    (\bar{\bW}_n)_{ij}:=\begin{cases}
    w^{(n)}_{i-j},& 0\leq i-j\leq \tau_n\\
    0, & \mbox{otherwise}.
    \end{cases}
\end{equation}
Note that the matrix $\bar\bW_n$ is a semi-infinite Toeplitz matrix whose diagonal and the first to the $\tau_n$-th sub-diagonals are the components of $\bw^{(n)}:=[w^{(n)}_0, w^{(n)}_1, \dots, w^{(n)}_{\tau_n}]^\top$.

We define $\bar{\cN}_n(\bx)$ by
\eqref{tilde-DNN} with $\tilde{\bW}_n$ being replaced by $\bar{\bW}_n$ and $\bar{\cN}_n(\bx)$ is an extension of $\cN_n(\bx)$, different from  $\tilde{\cN}_n(\bx)$. Then,  $\bar{\cN}_n(\bx)$, $n\in\bN$, satisfy the recursions \eqref{Step1-hat} and \eqref{Recursion-hat}.

\begin{definition}\label{CNN-NConv-Inf} 
We say that a sequence of vector-valued functions $\cN_n:\bR^s\to \bR^{m_n}$, $n\in \bN$, converges uniformly in $\DD$ to a function $\cN:\bR^s\to \ell_p(\bN)$ if the functions $\bar{\cN}_n:\bR^s\to \ell_p(\bN)$, $n\in \bN$, converges uniformly to $\cN$ in the space $C_{\ell_p(\bN)}(\DD)$. 
\end{definition}

We next study the uniform convergence of $\cN_n:\bR^s\to \bR^{m_n}$, $n\in \bN$. We first transfer a condition on the masks to the weight matrices.

\begin{lemma}\label{Assumption-on-Masks}
If  a sequence of masks $[w^{(n)}_0, w^{(n)}_1, \dots, w^{(n)}_k]\in\bR^{k+1}$, $n\in\bN$, satisfies the conditions that for each $k=0,1,\dots,\tau$, the sequence $w_k^{(n)}$, $n\in\bN$, converges in $\bR$ and
\begin{equation}\label{Mask-Condition}
    \lim_{n\to\infty}L\sum_{k=0}^\tau|w^{(n)}_k|<1,
\end{equation}
then 
$\lim_{n\to\infty}\|\bar\bW_n\|_{\ell_p(\bN^2)}$ exists and
\begin{equation}\label{CNN-BasicAssumption}
    \lim_{n\to\infty}L\|\bar\bW_n\|_{\ell_p(\bN^2)}<1.
\end{equation}
\end{lemma}
\begin{proof}
First, we prove that 
$\lim_{n\to\infty}\|\bar\bW_n\|_{\ell_p(\bN^2)}$ exists. Since for each $k=0,1,\dots,\tau$, the sequence $w_k^{(n)}$, $n\in\bN$, converges in $\bR$, it is a Cauchy sequence in $\bR$. We next show that the sequence $\bar\bW_n$, $n\in\bN$, is Cauchy in $\ell_p(\bN^2)$.
To this end, we suppose that $n, m\in\bN$ are arbitrary. A direct computation leads to 
\begin{equation}\label{barWCauchy1}
    \|\bar{\bW}_{n+m}-\bar{\bW}_n\|_{\ell_1(\bN)}=\sum_{k=0}^\tau|w_k^{(n+m)}-w_k^{(n)}|
\end{equation}
and
\begin{equation}\label{barWCauchyInf}
\|\bar{\bW}_{n+m}-\bar{\bW}_n\|_{\ell_\infty(\bN)}=\sum_{k=0}^\tau|w_k^{(n+m)}-w_k^{(n)}|.
\end{equation}
Again, by the interpolation theorem of Mitjagin \cite{Mit}, it follows from equations \eqref{barWCauchy1} and \eqref{barWCauchyInf} for all $p\in [1,\infty]$ that
\begin{equation}\label{barWCauchyInf_p}
\|\bar{\bW}_{n+m}-\bar{\bW}_n\|_{\ell_p(\bN)}\leq\sum_{k=0}^\tau|w_k^{(n+m)}-w_k^{(n)}|.
\end{equation}
Let $\epsilon>0$ be arbitrarily small. Since for each $k=0,1,\dots,\tau$, the sequence $w^{(n)}_k$ converges, there exists $N\in \bN$ such that for all $n, m>N$, the right-hand-side of \eqref{barWCauchyInf_p} is bounded by $\epsilon$, which ensures that $\bar\bW_n$, $n\in \bN$, is Cauchy. Moreover, since for all $n,m\in\bN$
$$
\left|\|\bar\bW_n\|_{\ell_p(\bN^2)}-\|\bar\bW_{n+m}\|_{\ell_p(\bN^2)}\right|\leq \|\bar\bW_n-\bar\bW_{n+m}\|_{\ell_p(\bN^2)},
$$
the sequence $\|\bar\bW_n\|_{\ell_p(\bN^2)}$, $n\in\bN$, is Cauchy in $\bR$. Therefore, this sequence has a limit in $\bR$.

Likewise, by direct computation, we find that
\begin{equation}\label{barNorm1}
    \|\bar{\bW}_n\|_{\ell_1(\bN^2)}=\sum_{k=0}^\tau|w_k^{(n)}|
\end{equation}
and
\begin{equation}\label{barNormInf}
\|\bar{\bW}_n\|_{\ell_\infty(\bN^2)}=\sum_{k=0}^\tau|w_k^{(n)}|.
\end{equation}
Again, by the interpolation theorem of Mitjagin, it follows from equations \eqref{barNorm1} and \eqref{barNormInf} for all $p\in [1,\infty]$ that
\begin{equation}\label{barNormInf_p}
\|\bar{\bW}_n\|_{\ell_p(\bN^2)}\leq\sum_{k=0}^\tau|w_k^{(n)}|.
\end{equation}
Convergence of the sequence $\|\bar{\bW}_n\|_{\ell_p(\bN^2)}$, $n\in\bN$, together with inequality \eqref{barNormInf_p} and
hypothesis \eqref{Mask-Condition} leads to \eqref{CNN-BasicAssumption}.
\end{proof}

We next establish the uniform boundedness of $\bar\cN_n(\bx)$, $n\in\bN$.

\begin{lemma}\label{CNN-Boundedness}
Suppose that $\sigma:\bR\to\bR$ is Lipschitz continuous with a Lipschitz constant $L$ and $\DD\subseteq \bR^s$ is bounded. If there exists a constant $c>0$ such that $\|\tilde\bbb_n\|_{\ell_p(\bN)}\leq c$, for all $n\in \bN$ and the sequence of masks $[w^{(n)}_0, w^{(n)}_1, \dots, w^{(n)}_k]\in\bR^{k+1}$, $n\in\bN$, satisfies condition that for each $k=0,1,\dots,\tau$, the sequence $w_k^{(n)}$, $n\in\bN$, converges in $\bR$ and \eqref{Mask-Condition}, then there exist positive constants $c_1, c_2$ such that  
\begin{equation}\label{CNN-BundofNn}
    \|\bar\cN_n(\bx)\|_{\ell_p(\bN)}\leq c_1c_2+Lcc_2+c_2\max\{\|\sigma({\bf 0}_\infty)\|_{\ell_p(\bN)}: j\in \bN_n\}, \ \ \mbox{for all}\ \ n\in \bN,\ \bx\in \DD.
\end{equation}
Moreover, there exists a constant $\rho>0$ such that
$\|\bar\cN_n(\bx)\|_{\ell_p(\bN)}\leq \rho$, for all $n\in\bN$ and $\bx\in\DD$, for $p\in[1,+\infty]$ if $\sigma(0)=0$ and for $p=+\infty$ if $\sigma(0)\neq 0$.
\end{lemma}
\begin{proof}
By Lemma \ref{Assumption-on-Masks}, we have that inequality \eqref{CNN-BasicAssumption} holds true. This together with Lemma \ref{BOUND-hat} with $\cP_\mu$ being the identity operator ensures that estimate \eqref{CNN-BundofNn} is satisfied. It follows from \eqref{CNN-BundofNn} that $\|\bar\cN_n(\bx)\|_{\ell_p(\bN)}$ is bounded for all $n\in\bN$ and $\bx\in\DD$, for $p\in[1,+\infty]$ if $\sigma(0)=0$ and for $p=+\infty$ if $\sigma(0)\neq 0$.
\end{proof}

With the above preparation, we are ready to present our second uniform convergence theorem for CNNs.

\begin{theorem}\label{CNN-UniformConvergence}
Suppose that $\sigma:\bR\to\bR$ is Lipschitz continuous with a Lipschitz constant $L$ and $\DD\subseteq \bR^s$ is an arbitrary bounded set. If the sequence $\bbb_n$, $n\in\bN$, converges, and for each $k=0,1,\dots, \tau$,   $\{w^{(n)}_k\}$ is convergent and satisfies hypothesis \eqref{Mask-Condition}
then the neural networks $\cN_n$ converge uniformly to a function in $C_{\ell_p(\bN)}(\DD)$ for all $p\in[1,\infty]$ if $\sigma(0)=0$ and for $p=+\infty$ if $\sigma(0)\neq 0$.

Moreover, if the sequences $\bbb_n$, and $w^{(n)}_k$, $n\in\bN$, for all $k=0,1,\dots, \tau$, converge exponentially, then the uniform convergence of $\cN_n$ to a function in $C_{\ell_p(\bN)}(\DD)$ is exponential,  for all $p\in[1,\infty]$ if $\sigma(0)=0$ and for $p=+\infty$ if $\sigma(0)\neq 0$.
\end{theorem}
\begin{proof}
Since for all $k=0,1,\dots, \tau$, the sequence  $w^{(n)}_k$, $n\in\bN$, satisfies hypothesis \eqref{Mask-Condition}, we see that Lemmas
\ref{Assumption-on-Masks} and \ref{CNN-Boundedness} hold true.
This theorem follows from Theorems \ref{Fixed-widths-hat} and \ref{Order-Fixed-widths-hat} with $\cP_\mu$ being the identity operator.
\end{proof}

Theorem \ref{CNN-UniformConvergence} may be further extended to CNNs with pooling and we leave this to the interested reader.

{\small
\bibliographystyle{amsplain}

}

\end{document}